\documentclass[11pt, twoside]{article}

\usepackage{amsmath, amssymb, amsthm}
\usepackage{mathtools}
\usepackage{enumitem}
\usepackage[margin=2.5cm]{geometry}
\usepackage{setspace}

\usepackage{dsfont}
\usepackage{relsize}

\usepackage{verbatim}
\usepackage{cases}

\usepackage{graphicx,float,subcaption, wrapfig}
\usepackage{algorithm}
\usepackage{algpseudocode}

\usepackage{hyperref}
\usepackage[svgnames]{xcolor}

\usepackage[sort,compress]{natbib}

\newtheorem{theorem}{Theorem}
\newtheorem{corollary}[theorem]{Corollary}
\newtheorem{proposition}[theorem]{Proposition}
\newtheorem{lemma}[theorem]{Lemma}

\theoremstyle{definition}

\newtheorem{remark}{Remark}

\newcommand{\RR}{\mathbb{R}}

\newcommand{\EE}{\mathbb{E}}

\newcommand{\norm}[1]{\left\lVert #1 \right\rVert}
\newcommand{\inner}[2]{\left\langle #1, #2 \right\rangle}

\newcommand{\argmin}{\mathop{\mathrm{argmin}}}

\newcommand{\tr}{\mathrm{tr}}

\newcommand{\tbt}[4]{\begin{bmatrix}#1 & #2 \\ #3 & #4\end{bmatrix}}
\newcommand{\tbo}[2]{\begin{bmatrix}#1 \\ #2\end{bmatrix}}
\newcommand{\obt}[2]{\begin{bmatrix}#1 & #2\end{bmatrix}}
\newcommand{\stbt}[4]{\begin{bsmallmatrix}#1 & #2 \\ #3 & #4\end{bsmallmatrix}}

\title{{\bf On the Role of Transformer Feed-Forward Layers in \\ Nonlinear In-Context Learning}}
\author{Haoyuan Sun, Ali Jadbabaie, Navid Azizan \\ \\Massachusetts Institute of Technology}
\date{}

\begin{document}

\renewcommand{\thefootnote}{\roman{footnote}}
\footnotetext[0]{Corresponding author: Haoyuan Sun (\texttt{haoyuans [at] mit [dot] edu}).}
\renewcommand{\thefootnote}{\arabic{footnote}}

\maketitle

\begin{abstract}
    Transformer-based models demonstrate a remarkable ability for \textit{in-context learning} (ICL), where they can adapt to unseen tasks from a few prompt examples without parameter updates. 
    Recent research has illuminated how Transformers perform ICL, showing that the optimal \textit{linear self-attention} (LSA) mechanism can implement one step of gradient descent for linear least-squares objectives when trained on random linear regression tasks.
    Building on this, we investigate ICL for \textit{nonlinear} function classes.
    We first prove that LSA is inherently incapable of outperforming linear predictors on nonlinear tasks, underscoring why prior solutions cannot readily extend to these problems.
    To overcome this limitation, we analyze a Transformer block consisting of LSA and feed-forward layers inspired by the \textit{gated linear units} (GLU), which is a standard component of modern Transformers. 
    We show that this block achieves nonlinear ICL by implementing one step of gradient descent on a polynomial kernel regression loss.
    Furthermore, our analysis reveals that the expressivity of a single block is inherently limited by its dimensions.
    We then show that a deep Transformer can overcome this bottleneck by distributing the computation of richer kernel functions across multiple blocks, performing block-coordinate descent in a high-dimensional feature space that a single block cannot represent.
    Our findings highlight that the feed-forward layers provide a crucial and scalable mechanism by which Transformers can express nonlinear representations for ICL.
\end{abstract}

\section{Introduction}
In recent years, \textit{large language models} (LLM) based on the Transformer architecture~\citep{vaswani2017attention} have achieved remarkable successes across numerous domains such as computer vision~\citep{dosovitskiy2020image}, speech recognition~\citep{radford2023robust}, multi-modal data processing~\citep{yang2023dawn}, and human reasoning tasks~\citep{achiam2023gpt}.
Beyond these practical achievements, Transformers have also demonstrated the extraordinary ability to perform \textit{in-context learning} (ICL)~\citep{brown2020language,wei2022emergent,min2021metaicl}.
In ICL, a model adapts to different tasks by leveraging a \textit{prompt} containing a short sequence of examples without requiring updates to its parameters.

The study of in-context learning has attracted significant attention in recent literature.
Formally, given a prompt composed of input-label pairs and a query, $(x_1, f(x_1), x_2, f(x_2), \dots,$ $x_n, f(x_n), x_{\textsf{query}})$, a Transformer model is said to learn a function class $\mathcal{F}$ in context if it can accurately predict $f(x_{\textsf{query}})$ for previously unseen functions $f \in \mathcal{F}$.
Empirical studies by \citet{garg2022can} have demonstrated the efficacy of Transformers in performing in-context learning for a range of function classes such as linear functions, decision trees, and ReLU neural networks.

Building on these insights, \citet{akyurek2022learning} and \citet{von2023transformers} established that Transformers can learn linear functions in context by implicitly implementing gradient descent for a linear regression objective on the prompt's examples and provided a construction that facilitates such a mechanism.
Extending this perspective, subsequent theoretical analysis by \citet{ahn2023linear,zhang2023trained,mahankali2023one} showed that under suitably defined distributions over the prompts, the optimal single-layer \textit{linear self-attention} (LSA) mechanism can effectively implement one step of gradient descent.

Despite these advancements, the theoretical understanding of ICL for \textit{nonlinear} function classes remains underdeveloped. 
In this work, we demonstrate that the feed-forward layers in Transformers is a significant component to advance the theoretical understanding of ICL for nonlinear functions.
As a representative case, we show that using standard feed-forward layer architectures, such as those based on the \textit{gated linear units} (GLU)~\citep{mnih2007three,dauphin2017language,shazeer2020glu}, enables Transformers to efficiently learn quadratic and polynomial functions in context.
This insight underscores a broader significance for feed-forward layers enabling rich expressive powers within Transformer models.

\subsection{Our contributions}
In this work, we investigate the pivotal role of feed-forward layers in enabling Transformers to learn nonlinear functions in context.
We show that Transformers with alternating linear attention and feed-forward layers can learn nonlinear representations in context by implementing gradient descent updates with respect to a collection of nonlinear features. 
The key contributions of our paper are as follows:
\begin{list}{{\tiny $\blacksquare$}}{\leftmargin=1em}
\setlength{\topsep}{-1pt}
\setlength{\itemsep}{-3pt}
    \item We establish the essential role of feed-forward layers in nonlinear ICL.
    In Section~\ref{sec:lsa-lower-bound}, we show that when the target functions are nonlinear, \textit{no} deep linear self-attention (LSA) network can achieve lower in-context learning loss than a linear least-squares predictor. 
    Then, in Section~\ref{sec:bilinear-construction}, motivated by modern Transformer architectures, we construct an example with simple parameter configurations where a Transformer block consisting of one GLU-like feed-forward layer and one LSA layer can implement one step of gradient descent with respect to a quadratic kernel.
    \item In Section~\ref{sec:approx-opt}, we analyze a Transformer model featuring one block of the LSA-GLU mechanism.
    We show that the parameterization given in Section~\ref{sec:bilinear-construction} is approximately optimal.
    Notably, we show that, regardless of the feed-forward layer architecture, the optimal feed-forward layer must compute a kernel that captures all nonlinear features present in the target function.  
    This implies a fundamental limitation where the expressive power of a single Transformer block is inherently constrained by its embedding dimension.
    \item In Section~\ref{sec:efficiency}, we show that a deep Transformer model can overcome this limitation.
    By distributing the computation of complex kernel functions across multiple feed-forward layers, deep Transformers can represent nonlinear features that are substantially richer than the capacity of individual blocks.
    In Section~\ref{sec:cubic-training}, we illustrate this approach with a concrete example in learning higher-order polynomial functions in context.
    This demonstrates how depth enables scalable in-context learning over richer function classes.
    Lastly, we substantiate these insights with numerical experiments throughout the section.
\end{list}
While our technical discussions focus on a GLU-like feed-forward layer for concreteness, the underlying principles naturally generalize to other choices of nonlinear feed-forward layers.
We discuss this broader applicability in Appendix~\ref{sec:extended-discussion}.

\subsection{Related works} 
Motivated by the successes of the Transformer architecture, there has been tremendous progress in understanding its algorithmic power.
Independent of the studies on ICL, many works sought to characterize the expressivity of Transformers from an approximation-theoretic perspective.
These studies established that Transformers can effectively approximate a wide range of functions and algorithms (e.g., \citet{perez2021attention,wei2022statistically,giannou2023looped,olsson2022context,bai2024transformers,edelman2022inductive,reddy2024mechanistic}).

In the context of ICL, numerous empirical studies have examined the Transformer's capabilities to learn a variety of tasks and settings in context, including linear regression~\citep{garg2022can}, discrete boolean functions~\citep{bhattamishra2023understanding}, representation learning~\citep{guo2023transformers}, reinforcement learning~\citep{lin2023transformers}, and learning causal structures~\citep{nichani2024transformers}.
To explain these observations, \citet{akyurek2022learning,dai2023can,li2023transformers} proposed that Transformers can learn in context by implementing learning algorithms.
In particular, for linear regression, \citet{von2023transformers} proposed that Transformers implement gradient descent with zero initialization and provided a simple construction for a single-layer linear attention model.
Subsequently, \citet{ahn2024transformers,mahankali2023one,zhang2023trained,wu2024many} showed that the optimal single-layer linear self-attention with respect to some suitably chosen in-context learning loss closely corresponds to the construction proposed by~\citet{von2023transformers}.
Additionally, \citet{chen2024training} studied the training dynamics of applying softmax attention to linear ICL tasks, and they showed that for long context $n$, the solution attained by softmax attention is approximately equivalent to that of linear attention.

While theoretical analysis of attention for nonlinear ICL has gathered significant attention in the literature, existing works remain limited.
First, \citet{cheng2023transformers} show that attention with task-specific nonlinear activations can implement gradient descent in a function space induced by those activations.
But this requires redesigning the attention mechanism for each function class, limiting generality.
Another line of work \citep{huang2023context,yang2024context} studies how multi-head softmax attention can learn functions represented by a \textit{fixed} nonlinear feature map.
However, these works impose some significant limitations on the input distribution.
\citet{huang2023context} assumes that the inputs $x_i$ are drawn from an orthonormal set of vectors, while \citet{yang2024context} assumes the list of $x_i$'s is a fixed dictionary identical across all prompts. In contrast, our work assumes inputs are drawn i.i.d. from $\RR^d$, a significantly less restrictive assumption.

There is also an emerging area of study which studies the importance of feed-forward layers in enabling ICL.
\citet{zhang2024context} showed that incorporating an MLP feed-forward layer allows for Transformer to implement gradient descent with learnable \textit{nonzero} initialization, but their result is still confined to linear ICL.
Also, \citet{oko2024pretrained} showed that a model consisting of two-layer ReLU network and linear attention can perform ICL on \textit{low dimensional} generalized linear models, where the dimension of the target function's parameter space has to be much smaller than both the Transformer's embedding dimension and the ReLU network width.
Our work overcomes this shortcoming by investigating a deep Transformer architecture that can learn rich nonlinear functions in-context even when a single block lacks the capacity to represent such features.

\section{Problem Setting}
\label{sec:settings}

\subsection{In-context learning}
\label{sec:in-context}
We are interested in the in-context learning problem, where a \textit{prompt} consists of a sequence of input-label pairs $\{(x_i, y_i)\}_{i=1}^n$ along with a query $x_{\textsf{query}}$, where $y_i = f(x_i)$ for an unknown \textit{target function} $f : \RR^d \to \RR$.
Given this prompt, our goal is to make a prediction $\hat{y}$ so that $\hat{y} \approx y_{\textsf{query}} = f(x_{\textsf{query}})$.
In this paper, the prompt is represented as a matrix $Z \in \RR^{(\bar{d}+1) \times (n+1)}$:
\begin{equation}
\label{equ:prompt-form}
    Z =
    \begin{bmatrix}
        1 & 1 & \cdots & 1 & 1 \\
        x_1 & x_2 & \cdots & x_n & x_{n+1} := x_{\textsf{query}} \\
        \mathbf{0} & \mathbf{0} & \cdots & \mathbf{0} & \mathbf{0} \\
        y_1 & y_2 & \cdots & y_n & 0
    \end{bmatrix},
\end{equation}
where and the bottommost zero hides the unknown quantity $y_\textsf{query}$ that we seek to predict.
Similar to the work by \citet{bai2024transformers}, we include $\bar{d} - d - 1$ rows of zeros to embed the input in a higher-dimensional space.
Also, we include a row of ones to accommodate inhomogeneity in the nonlinear target function $f$.

We sample the prompts by first drawing i.i.d. Gaussian inputs $x_1, \dots, x_{n+1} \sim \mathcal{N}(0, \mathbf{I}_d)$.
In this work, we consider random \textit{quadratic} target functions of the form:
\begin{align*}
    f(x) = w_0 + \sum_{i=1}^d w_i x_i + \sum_{1 \le i \le j \le d} w_{ij} x_i x_j, \quad \text{where } w_i, w_{ij} \overset{i.i.d.}{\sim} \mathcal{N}(0, 1).
\end{align*}
For a symmetric matrix $M$, we define $\mathrm{svec}(M)$ as the vector obtained by flattening $M$ but which does not repeat the off-diagonal entries and scales them by a factor of 2.
And let $\mathrm{smat}(\cdot)$ be the inverse of this operation.
Then, in compact notation, we equivalently write:
\[f(x) = \obt{1}{x} W \tbo{1}{x}, \quad \text{for } \mathrm{svec}(W) \sim \mathcal{N}(0, \mathbf{I}),\]
where $W$ is a $(d+1) \times (d+1)$ random matrix and $\mathrm{svec}(W)$ has length $\binom{d+1}{2} + (d+1) = \binom{d+2}{2}$.
Given this class of target functions, we set $y_i = f(x_i)$.
Then, prompts $Z$ are sampled based on the randomness in the inputs $x_i$ and the coefficients $w$ of the target quadratic function.

\subsection{Transformer architecture}
Given a prompt $Z \in \RR^{D \times N} := \RR^{(\bar{d}+1) \times (n+1)}$, the \textit{Transformer} model is a sequence-to-sequence mapping $\RR^{D \times N} \mapsto \RR^{D \times N}$ that consists of self-attention and feed-forward layers:
\begin{list}{{\tiny $\blacksquare$}}{\leftmargin=1em}
\setlength{\itemsep}{1pt}
    \item \textsc{Attention Layer:} The most crucial component of Transformers is the \textit{self-attention} mechanism~\citep{bahdanau2014neural}, which can be written as
    
    \begin{equation*}
        \hspace{-0.5em} \mathsf{Attention}(Z) = Z + \frac{1}{n} \mathbf{W_V} Z M \, \sigma \bigg((\mathbf{W_K}Z)^\top(\mathbf{W_Q}Z)\bigg).
    \end{equation*}
    Here, $\sigma$ is typically the \textit{softmax} function, $\mathbf{W_Q}, \mathbf{W_K}, \mathbf{W_V} \in \RR^{D \times D}$ are learnable matrix parameters, and
    mask matrix $M = \mathrm{diag}(\mathbf{I}_n, 0)$ erases the attention score with the last column,
    due to the asymmetry that $y_{\textsf{query}}$ is not part of the prompt.
    Following the conventions established by \citet{von2023transformers}, we consider a variant called \textit{linear self-attention} (LSA), which omits the softmax activation and reparameterizes the weights as $\mathbf{P} := \mathbf{W_V}, \mathbf{Q} := \mathbf{W_K}^\top \mathbf{W_Q}$:
    \begin{equation}
    \label{equ:attention}
        \mathsf{attn}(Z) = Z + \frac{1}{n} \mathbf{P} Z M \, \bigg(Z^\top\mathbf{Q}Z\bigg).
    \end{equation}
    We note that LSA is practically relevant --- \citet{ahn2023linear} found that the optimization landscape of LSA closely mirrors that of full Transformers.
    Furthermore, variants of the linear attention mechanism have gained significant popularity in improving the computational efficiency of state-of-the-art LLMs~\citep{dao2024transformers,yang2025gated,li2025minimax}.
    \item \textsc{Feed-forward Layer:} The original Transformer model \citep{vaswani2017attention} employs a \textit{multi-layer perceptron} (MLP) for the feed-forward layer. 
    However, recent studies have shown that variants of the \textit{gated linear unit} (GLU) offer improved performance in large language models~\citep{shazeer2020glu}.
    A GLU is the element-wise product of two linear projections and was first proposed in~\citet{dauphin2017language}.
    A commonly used variant is the \textit{SwiGLU} layer, where one of the projections passes through a \textit{swish} nonlinearity.
    In this paper, we consider the simplest GLU variant without any nonlinear activations, which is referred to as a ``bilinear'' layer~\citep{mnih2007three}.
    For matrix parameters $\mathbf{W}_0, \mathbf{W}_1 \in \RR^{\bar{d} \times \bar{d}}$, we define
    \begin{equation}
        \label{equ:bilinear}
        \mathsf{bilin}(Z) = Z + \left(\tbt{\mathbf{W}_0}{\mathbf{0}}{\mathbf{0}}{0} Z\right) \odot \left(\tbt{\mathbf{W}_1}{\mathbf{0}}{\mathbf{0}}{0} Z\right),
    \end{equation}
    where $\odot$ denotes the Hadamard product.
    And for computing this layer, we mask the last row because of the asymmetry that the final label $y_{\mathsf{query}}$ is absent from the prompt.
    For the rest of the paper, we refer to \eqref{equ:bilinear} as the \textit{bilinear feed-forward layer} to avoid any ambiguity.
\end{list}

For a Transformer $\mathsf{TF}$, let $Z^{(\ell)}_{\mathsf{TF}}$ (or simply $Z^{(\ell)}$ when the Transformer is clear from the context) denote the output of its $\ell$th layer.
In this paper, we focus on two types of Transformer models.
\begin{list}{{\tiny $\blacksquare$}}{\leftmargin=1em}
\setlength{\itemsep}{-3pt}
    \item A \textbf{linear Transformer} consists solely of LSA layers.
    A linear Transformer $\mathsf{TF_{lin}}$ with $L$ layers is defined as:
    \[Z^{(0)}_\textsf{lin} = Z, \quad Z^{(\ell+1)}_\textsf{lin} = \mathsf{attn}(Z^{(\ell)}_\textsf{lin}), \quad \ell = 0, \dots, L-1.\]
    \item A \textbf{bilinear Transformer} consists of alternating LSA and bilinear feed-forward layers.
    A bilinear Transformer $\mathsf{TF_{bilin}}$ with $2L$ layers follows the forward rule:
    \begin{align*}
        Z^{(0)}_\textsf{bilin} = Z, \;\; Z^{(2\ell+1)}_\textsf{bilin} = \mathsf{bilin}(Z^{(2\ell)}_\textsf{bilin}),
        Z^{(2\ell+2)}_\textsf{bilin} = \mathsf{attn}(Z^{(2\ell+1)}_\textsf{bilin}), \quad \ell = 0, \dots, L-1.
    \end{align*}
    We note that this differs slightly from the canonical definition of Transformers in that we have swapped the order of the attention and feed-forward layers.
    This choice is to simplify our construction in Section~\ref{sec:bilinear-construction} and does not affect the expressive power of the Transformer.
\end{list}

\subsection{In-context learning objective}
We now define the objective of the in-context learning problem.
Given an $L$-layer Transformer $\mathsf{TF}$, we let its prediction $\hat{y}_{\mathsf{TF}}$ be the $(\bar{d}+1, n+1)$th entry of its final layer, i.e. $\hat{y}_{\mathsf{TF}} = Z_{\mathsf{TF}}^{(L)}[\bar{d}+1, n+1]$.
Then, its \textit{in-context learning loss} is the expected squared error with respect to the distribution of the prompts $Z$:
\begin{equation}
    \mathcal{L}(\mathsf{TF}) = \EE_{Z}[(\hat{y}_{\mathsf{TF}} + y_{\textsf{query}})^2],
\end{equation}
so that the Transformer's prediction $\hat{y}_{\mathsf{TF}}$ approximates $-y_\textsf{query}$.
Our objective is to find Transformers that minimize this loss $\mathcal{L}$.
 
\section{Solving Quadratic In-context Learning}
\label{sec:nonlinear}
Recent literature has explored the ability of linear self-attention (LSA) mechanisms to learn linear function classes in context. 
In this section, we demonstrate that LSA is inherently limited when it comes to nonlinear in-context learning (ICL). 
To address this limitation, inspired by modern Transformer architectures, we study a mechanism that combines LSA with bilinear feed-forward layers, enabling the model to solve quadratic ICL through kernel regression. 
Our approach is quite general, and as we show in Section~\ref{sec:cubic-training}, it can be extended to handle ICL of higher-degree polynomials.

\subsection{Limitations of linear self-attention in nonlinear ICL}
\label{sec:lsa-lower-bound}

Prior works \citep{mahankali2023one,ahn2024transformers,zhang2023trained} have shown that a single layer of linear self-attention (LSA) can learn linear functions in context by implementing one step of gradient descent with respect to a linear least-squares objective.
Formally, the optimal one-layer LSA model that minimizes the linear ICL loss has weights of the following form:
\begin{equation}
\label{equ:linear-attn-weights}
    \mathbf{P} = \tbt{0_{d \times d}}{0}{0}{1}, \mathbf{Q} = \tbt{\Gamma}{0}{0}{0},
\end{equation}
and its prediction given prompt $Z$ can be expressed as
\[\hat{y} = \frac{1}{n} \sum_{i=1}^n (x_\textsf{query}^\top \Gamma x_i) y_i.\]
This is equivalent to the prediction obtained after one step of preconditioned gradient descent with zero initialization and linear regression objective $L(w) = \frac{1}{2n} \sum_i (w^\top x_i + y_i)^2$.

We note that $L$-layers of LSA is a degree $3^L$ polynomial in $Z$. Also, to learn linear functions in context, the optimal attention weights are sparse.
Thus, a natural question arises: can we leverage these unused weights in LSA to learn richer classes of functions in context? 
The following proposition answers this question by showing that, even with multiple layers, a linear Transformer consisting of only LSA layers \textit{cannot} effectively learn a nonlinear target function in context.

\vspace{-0.25em}
\begin{proposition}
\label{thm:lsa-lower-bound}
    Consider any fixed target function $f$ and inputs $x_1, \dots, x_{n+1} \overset{i.i.d.}{\sim} \mathcal{N}(0, \Sigma)$.
    The in-context learning prompts are sampled according to the form as described in \eqref{equ:prompt-form}.
    Then, for any linear Transformer $\mathsf{TF_{lin}}$ (regardless of the number of layers), its expected prediction loss satisfies
    \vspace{-0.25em}
    \[\EE_x[(\hat{y}_{\mathsf{TF}_\mathsf{lin}} + y_\textsf{query})^2] \ge \min_{\alpha, \beta} \EE_x[(\inner{x_\textsf{query}}{\beta} + \alpha + y_\textsf{query})^2].\]
\end{proposition}

This result reveals the inherent limitation: a linear Transformer's prediction performance is bounded by that achieved by solving linear regression, preventing it from effectively learning a nonlinear target function in context.
For a complete proof of this result, please see Appendix~\ref{sec:proof-lsa-lower-bound}.

\subsection{Quadratic ICL with bilinear feed-forward layers}
\label{sec:bilinear-construction}

In order to overcome the limitation of LSA in handling nonlinear functions, drawing inspiration from modern Transformers, we study a mechanism that augments the linear Transformer architecture with bilinear feed-forward layers.
In particular, we consider two layers in a bilinear Transformer and define a \textit{bilinear Transformer block} $\mathsf{BTFB}$ as the composition of a bilinear feed-forward layer followed by a linear self-attention layer, i.e.,
$\mathsf{BTFB}(Z) := \mathsf{attn} \circ \mathsf{bilin}(Z).$

In a $2L$-layer bilinear Transformer $\mathsf{TF_{bilin}}$, there are $L$ bilinear Transformer blocks.
We claim that for sufficiently large embedding dimension $\bar{d}$ and context length $n$, a single bilinear Transformer block can learn quadratic target functions in context.

\vspace{-0.25em}
\begin{theorem}
\label{thm:bilinear-construction}
    Given the quadratic in-context learning problem as described in Section~\ref{sec:in-context} and embedding dimension $\bar{d} = \binom{d+2}{2}$, there exists a bilinear Transformer block $\mathsf{BTFB}$ such that
    \[\mathcal{L}(\mathsf{BTFB}) = \frac{(d+2)(d+1) + d}{2n} \in O(d^2/n).\]
\end{theorem}

A formal proof of this statement can be found in Appendix~\ref{sec:proof-bilinear-construction}.
At a high level, we can achieve this bound in two steps.
\begin{list}{{\tiny $\blacksquare$}}{\leftmargin=1em}
\setlength{\topsep}{-1pt}
\setlength{\itemsep}{-1pt}
    \item The bilinear feed-forward layer computes quadratic features, expanding the input space to capture every quadratic term in the target function.
    \item The labels $y_i$'s are then linear in these quadratic features, which allows the attention layer to solve for a least-squares objective in context.
\end{list}

\underline{Step 1:} To understand how the bilinear feed-forward layer computes the necessary quadratic terms, we recall the definition \eqref{equ:bilinear}.
Specifically, we denote $Z^+ := \mathrm{diag}(\mathbf{W}_0, 0) Z \odot \mathrm{diag}(\mathbf{W}_1, 0) Z$.
Let $\mathbf{W}_{0, j}$ be the $j$th row of $\mathbf{W}_0$, then we have
\begin{equation}
\label{equ:bilinear-kernel-formula}
    Z^+[i, j] = (\mathbf{W}_{0, j} x_i) \cdot (\mathbf{W}_{1, j} x_i) = x_i^\top ({\mathbf{W}_{0, j}}^\top \mathbf{W}_{1, j}) x_i,
\end{equation}
which is quadratic in $x_i$.
Then, it is not difficult to see that there exist bilinear layer where
\begin{flalign}
    &\textsf{bilin}(Z) =
    \begin{bmatrix}
        \bar{x}_1 & \bar{x}_2 & \cdots & \bar{x}_n & \bar{x}_\textsf{query} \\
        y_1 & y_2 & \cdots & y_n & 0
    \end{bmatrix}, \text{ with} \label{equ:kernel} \\
    &\bar{x}_i = (1, x_i[1], \dots, x_i[d], x_i[1]^2-1, \dots, x_i[1]x_i[d], x_i[2]^2-1, \dots, x_i[2]x_i[d], \dots, x_i[d]^2-1), \notag
\end{flalign}
which contains every quadratic monomial in $x_i$. So, the label $y_i$ is linear in $\bar{x}_i$.

\underline{Step 2:} Given the bilinear layer that implements the mapping $x_i \mapsto \bar{x}_i$ \eqref{equ:kernel}, the subsequent attention layer receives a prompt consisting of input-label pairs $(\bar{x}_i, y_i)$, where $y_i$ is linear in $\bar{x}_i$.
With the same choice of weights as in \eqref{equ:linear-attn-weights}, the LSA layer performs linear regression on pairs $(\bar{x}_i, y_i)$, and thus the final prediction of the bilinear Transformer block can be written as
\[\hat{y} = \frac{1}{n} \sum_{i=1}^n \bar{x}_\textsf{query}^\top \Gamma \bar{x}_i y_i.\]
We note that this solution corresponds to solving for kernel regression with a kernel function $\kappa(x_\textsf{query}, x_i) = \bar{x}_\textsf{query}^\top \Gamma \bar{x}_i$.
Therefore, the bilinear layer enables a bilinear Transformer block to express this kernel space, effectively implementing gradient descent on kernel regression.

\subsection{Analysis of the optimal bilinear Transformer block }
\label{sec:approx-opt}

In Section~\ref{sec:bilinear-construction}, we showed that a bilinear Transformer block learn quadratic functions in-context.
In this section, we investigate what choices of the Transformer weights can minimize the quadratic ICL loss.
In particular, we study the solution attained by the following two-stage procedure:
\begin{algorithm}
\caption{Finding the optimal bilinear Transformer block}
\label{alg:two-stage-opt}
\begin{algorithmic}[1]
    \State Initialize the attention weights with $\mathbf{P}_0 = \stbt{0_{d \times d}}{0}{0}{1}, \mathbf{Q}_0 = \stbt{\Gamma}{0}{0}{0},$ for any full-rank diagonal $\Gamma$.
    \State \textbf{Stage 1:} With attention weights frozen, we find the optimal bilinear layer weights
    \[\mathbf{W}_1^\star, \mathbf{W}_2^\star \gets \argmin \mathcal{L}(\mathsf{BTFB}) \quad \text{ s.t. } \mathbf{P} = \mathbf{P}_0, \mathbf{Q} = \mathbf{Q}_0.\]
    \State \textbf{Stage 2:} With bilinear layer weights frozen, we find the optimal attention weights
    \[\mathbf{P}^\star, \mathbf{Q}^\star \gets \argmin \mathcal{L}(\mathsf{BTFB}) \quad \text{ s.t. } \mathbf{W}_1 = \mathbf{W}_1^\star, \mathbf{W}_2 = \mathbf{W}_2^\star.\]
\end{algorithmic}
\end{algorithm}

Our analysis of this procedure can be broken down into two steps:

\begin{list}{{\tiny $\blacksquare$}}{\leftmargin=1em}
\setlength{\topsep}{-1pt}
\setlength{\itemsep}{1pt}
    \item We characterize the property of the optimal bilinear layer.
    We begin by showing that any Transformer block incurs a large ICL loss unless its feed-forward layer generates all necessary quadratic features.
    Then, by comparing this lower bound with the loss achieved by our construction in Section~\ref{sec:bilinear-construction}, we conclude that the optimal bilinear layer must implement a minimal representation of these quadratic features, i.e. the kernel function in (7), up to a change of coordinates.
    \item Given this choice of the bilinear layer, we find attention weights that minimize the ICL loss. In particular, the optimal weights approximately corresponds to the negative inverse of the features' covariance matrix.
    Consequently, as the context length $n \to \infty$, the resulting predictor converges to the ordinary least squares (OLS) estimator computed over the quadratic feature space.
\end{list}

\underline{Step 1:} We first show that features generated by the kernel function \eqref{equ:kernel} is a minimal representation of the quadratic feature.
To this end, we construct an inner product space in which the features computed by \eqref{equ:kernel} are linearly independent.
This argument is formalized as follows:
\begin{lemma}
\label{thm:quadratic-basis}
    Given random vectors $x \sim \mathcal{N}(0, \mathbf{I}_d)$ in $\RR^d$ and define an inner product on random variables as $\inner{u}{v} := \EE_x[uv]$. Then, the following set of $\binom{d+2}{2}$ random variables
    \begin{equation}
    \label{equ:quadratic-basis}
        \left\{v_{ij} \vphantom{x^2}\right\}_{0 \le i \le j \le d} :=
        \begin{cases}
                1 & \text{if  } i = j = 0, \\
                x[j] & \text{if  } i= 0, j \in \{1, \dots, d\}, \\
                \frac{1}{\sqrt{2}} (x[i]^2-1) & \text{if  } i = j \in \{1, \dots, d\}, \\
                x[i]x[j] & \text{if  } 1 \le i < j \le d,
        \end{cases}
    \end{equation}
    are orthonormal under this inner product space.
\end{lemma}

Note that \eqref{equ:quadratic-basis} is an linear map of the features in \eqref{equ:kernel}.
Consider a bilinear layer whose output has the form
\[
\textsf{bilin}(Z) =
    \begin{bmatrix}
        \widetilde{x}_1 & \widetilde{x}_2 & \cdots & \widetilde{x}_n & \widetilde{x}_\textsf{query} \\
        y_1 & y_2 & \cdots & y_n & 0
    \end{bmatrix}.
\]
This implements a kernel mapping $x \mapsto \widetilde{x}$, and $\widetilde{x}$ is a $\bar{d}$-dimensional random variable.
Thus, we can think of $\widetilde{x}$ as a collection of $\bar{d}$ vectors in this inner product space of random variables.
Then, we can apply a linear algebra argument to lower bound the ICL loss whenever $\widetilde{x}$ does not span \eqref{equ:kernel}.
\begin{proposition}
\label{thm:embed-lower-bound}
    Given the quadratic ICL problem as described in Section~\ref{sec:in-context},
    we consider a bilinear Transformer block $\mathsf{BTFB}$ whose bilinear layer implements a kernel mapping $x \mapsto \widetilde{x}$.
    For the inner product space defined in Lemma~\ref{thm:quadratic-basis}, if the orthogonal projection of $\widetilde{x}$ onto the span of \eqref{equ:quadratic-basis} has dimension $\widetilde{d} < \binom{d+2}{2}$, then the in-context learning loss is lower bounded by
    $\mathcal{L}(\mathsf{BTFB}) \ge \binom{d+2}{2} - \widetilde{d}.$
\end{proposition}

Because our construction in Section~\ref{sec:bilinear-construction} would achieve ICL loss at most $O(1/n)$, for sufficiently large $n$, the optimal bilinear Transformer block must have a feed-forward layer that has sufficient  capacity to implement the same kernel as \eqref{equ:kernel}, up to an invertible linear map in the inner product space.
The full proof of these statements can be found in Appendices~\ref{sec:proof-quadratic-basis} and~\ref{sec:proof-embed-lower-bound}.

\begin{remark}
    As an immediate consequence of Proposition~\ref{thm:embed-lower-bound}, if the prompt's embedding dimension $\bar{d}$ is less than $\binom{d+2}{2}$, then the ICL loss is lower bounded by $\mathcal{L}(\mathsf{BTFB}) \ge \binom{d+2}{2} - \bar{d}.$
\end{remark}

\begin{remark}
This lower bound is a fundamental limitation of the self-attention mechanism and applies to any choice of feed-forward network. We will further discuss this point in Appendix~\ref{sec:extended-discussion}.
\end{remark}

\underline{Step 2:} Recall from Section~\ref{sec:bilinear-construction}, we leveraged the fact that the label $y_i$'s are linear in the features $\bar{x}_i$'s, and therefore the attention layer can be thought of as implementing linear ICL on this kernel feature space.
However, to find the optimal attention weights, we cannot directly appeal to the prior literature such as \cite{ahn2024transformers} because the distribution of $\bar{x}_i$ is \textit{not Gaussian}.
While there is no closed-form solution in our setting, we can still approximate the optimal weights.
And similar to the linear ICL setting, the optimal attention weights converge to the negative inverse covariance of the feature vectors $\bar{x}_i$'s in the limit as the prompt length $n \to \infty$.

\begin{theorem}
\label{thm:approx-opt}
    Consider a linear Transformer block $\mathsf{BTFB}$ and fix its bilinear layer to correspond the kernel $x \mapsto \bar{x}$ as described in \eqref{equ:kernel}. 
    Then, the optimal choice of self-attention layer that minimizes the in-context learning loss $\mathcal{L}(\mathsf{BTFB})$ has weights
    \vspace{-0.5em}
    \[\mathbf{P} = \tbt{0_{d \times d}}{0}{0}{1}, \mathbf{Q} = \tbt{\Gamma^\star}{0}{0}{0},\]
    where $\Gamma^\star$ satisfies
    \vspace{-0.5em}
    \[\norm{\Gamma^\star + \EE[\bar{x} \bar{x}^\top]^{-1}}_F \le n^{-1/2} \sqrt{\frac{(d+2)(d+1) + d}{d+2-\sqrt{d^2+4d}}} \in O\left(\frac{d^{3/2}}{\sqrt{n}}\right).\]
\end{theorem}

When the prompt length $n$ is large, the optimal weight $\Gamma^*$ closely approximates the negative inverse of the feature covariance matrix $\EE[\bar{x}\bar{x}^\top]^{-1}$.
We also note that $\EE[\bar{x}\bar{x}^\top]^{-1}$ is the limit of the Gram matrix $(\frac{1}{n}\sum_i \bar{x}\bar{x}^\top)^{-1}$ as $n\to\infty$.
Therefore, given sufficiently many prompt examples, the optimal attention layer in the bilinear Transformer block approximates the ordinary least-squares solution with respect to the quadratic features.
The proof of this result can be found in Appendix~\ref{sec:proof-approx-opt}.

By combining the results in Theorem~\ref{thm:bilinear-construction}, Proposition~\ref{thm:embed-lower-bound} and Theorem~\ref{thm:approx-opt}, we can characterize the solution found by Algorithm~\ref{alg:two-stage-opt}.
\begin{corollary}
    For sufficiently large context length $n \in \Omega(d^3)$, the bilinear Transformer block attained by Algorithm~\ref{alg:two-stage-opt} satisfies the following properties:
    \begin{list}{{\tiny $\blacksquare$}}{\leftmargin=1em}
    \setlength{\topsep}{-1pt}
    \setlength{\itemsep}{-1pt}
        \item The solution to bilinear layer weights corresponds to a mapping $x \mapsto Mv$, where $M$ is invertible and $v$ is the vector of features defined in \eqref{equ:quadratic-basis}.
        \item The solution to linear attention layer weights approximately solves for the ordinary linear-square estimator with respect to the nonlinear features computed by the bilinear layer.
        \item The in-context learning loss of the learned bilinear Transformer block is at most $O(1/n)$.
    \end{list}
\end{corollary}
 
\section{Deep Bilinear Transformers}
\label{sec:efficiency}
While the approach introduced in Section~\ref{sec:bilinear-construction} demonstrates that a single bilinear Transformer block can learn quadratic functions in context, it faces several key limitations.
As formalized in Proposition~\ref{thm:embed-lower-bound}, the expressivity of one Transformer block is fundamentally bottle-necked by the features produced by its bilinear layers.
In particular, to perform quadratic ICL over $d$ variables, the Transformer's embedding dimension $\bar{d}$ must scale as $\Theta(d^2)$.
Because the computational costs of self-attention grows quadratically in $\bar{d}$, increasing the model width is prohibitively expensive.
Moreover, for more complex ICL tasks, the feed-forward layer in a single block may lack sufficient capacity to generate rich nonlinear features.
To overcome these shortcomings, we study a \textit{deep} bilinear Transformer, in which each Transformer block performs one step of block-coordinate descent.
This approach distributes the computation of nonlinear features across multiple layers, enabling the model to achieve greater expressivity without inflating the size of individual Transformer blocks.
We also validate these insights through a series of numerical experiments.

\subsection{Deep bilinear Transformers as block-coordinate descent solvers}
\label{sec:block-construction}

To address the aforementioned limitations of a single bilinear Transformer block, we consider a \textit{deep} bilinear Transformer network by stacking together multiple bilinear Transformer blocks.
In particular, we show that in a deep bilinear Transformer, we can split the calculations required by a kernel space gradient descent update across multiple blocks.
Thus, we can solve for a quadratic ICL task with embedding space $\bar{d}$ less than the total number of features $\binom{d+2}{2}$.

Recall from Section~\ref{sec:bilinear-construction} that a single bilinear Transformer block corresponds to one step of preconditioned gradient descent with respect to a quadratic kernel.
In particular, for the weights $w \in \RR^{\binom{d+2}{2}}$, we define the quadratic function:
\[\hspace{2em} f(x; w) = \obt{1}{x} \mathrm{smat}(w) \tbo{1}{x}.\]
Then, given the prompt's examples $(x_1, y_1), \dots, (x_n, y_n)$, the Transformer's prediction corresponds to a gradient descent step with some preconditioner $\Gamma$ so that
\begin{align*}
\hat{y} = f(x^{\textsf{query}}; w^+), \quad w^+ = \Gamma \nabla_w[L(w)]_{w = 0}, \quad L(w) = \frac{1}{2n} \sum_{i=1}^n (f(x_i; w) + y_i)^2.
\end{align*}
This equivalence requires the bilinear layer to output all quadratic features.
If the embedding dimension is insufficient to fit all of them, we instead construct bilinear layers that only output a subset of these quadratic monomials.
Then, the subsequent attention layer selectively performs a partial gradient descent update with respect to only this subset of quadratic monomials.
This means that the bilinear Transformer across all of its layers implements block-coordinate descent step by step.

Specifically, we define a \textit{block-coordinate descent} update, where for subsets of the quadratic features indexed by $b_\ell \subseteq \{1, \dots, \binom{d+2}{2}\}$, we perform gradient updates only along the coordinates corresponding to $b_\ell$:
\begin{equation}
\label{equ:block-update}
    w^{(\ell)}_j = 
    \begin{cases}
        w^{(\ell-1)}_j -\eta^{(\ell)} \partial_{w_j}[L(w)]_{w = w^{(\ell-1)}} & \text{if } j \in b_\ell, \\
        w^{(\ell-1)}_j & \text{if } j \not\in b_\ell.
    \end{cases}
\end{equation}
In our problem, the $\ell$-th feature subset $b_\ell$ corresponds to the quadratic terms contained in the columns of the $\ell$-th bilinear layer's output.

\begin{theorem}
\label{thm:block-construction}
    For the quadratic ICL problem as described in Section~\ref{sec:in-context}, there exists a $\mathsf{TF_{bilin}}$ with $\bar{d} = 2d+1$ implementing block-coordinate descent over a quadratic regression in the sense that for $\hat{y}^{(\ell)} := Z_{\mathsf{TF_{bilin}}}^{(2\ell)}[\bar{d}+1, n+1]$, we have
    \[\hat{y}^{(\ell)} = \obt{1}{x_{\mathsf{query}}} W^{(\ell)} \tbo{1}{x_{\mathsf{query}}}, \]
    where $W^{(\ell)}$ are iterates of a block-coordinate descent update \eqref{equ:block-update}.
\end{theorem}
Note that there are many possible choices for the feature subsets, as long as every index is visited, i.e., $\bigcup_{\ell} b_\ell = \{1, \dots, \binom{d+2}{2}\}$.
In the proof of this Theorem in Appendix~\ref{sec:app-block-construction}, we choose the subsets $b_\ell$ to be set of indices in $w$ corresponding to the following quadratic features:
\begin{equation}
\label{equ:blocks}
    (1, x_i[1], \dots, x_i[d], x_i[1] x_i[\ell], \dots, x_i[d] x_i[\ell]),
\end{equation}
for $\ell=1, \dots, d$, and repeat this cycle thereafter.

To illustrate the block-coordinate descent update dynamics that we proposed in Section~\ref{sec:block-construction}, we conduct a numerical experiment comparing the quadratic ICL performance of linear Transformers without feed-forward layers and bilinear Transformers of various depths.
We set the problem dimensions to be $d \in \{3,4\}$, and $\bar{d}=12$.
Note that when $d=4$, there are $\binom{6}{2} = 15$ distinct monomials in the quadratic function, which is larger than our embedding space $\bar{d}$.
But when $d=3$, there are $\binom{5}{2} = 10$ distinct monomials, which is smaller than $\bar{d}$.
The training prompts have a fixed length of $n=200$ examples plus one query, and they are drawn from a distribution of quadratic function as described in Section~\ref{sec:in-context}.

\begin{figure}[!tb]
    \centering
    \begin{subfigure}{0.32\textwidth}
        \includegraphics[width=\textwidth]{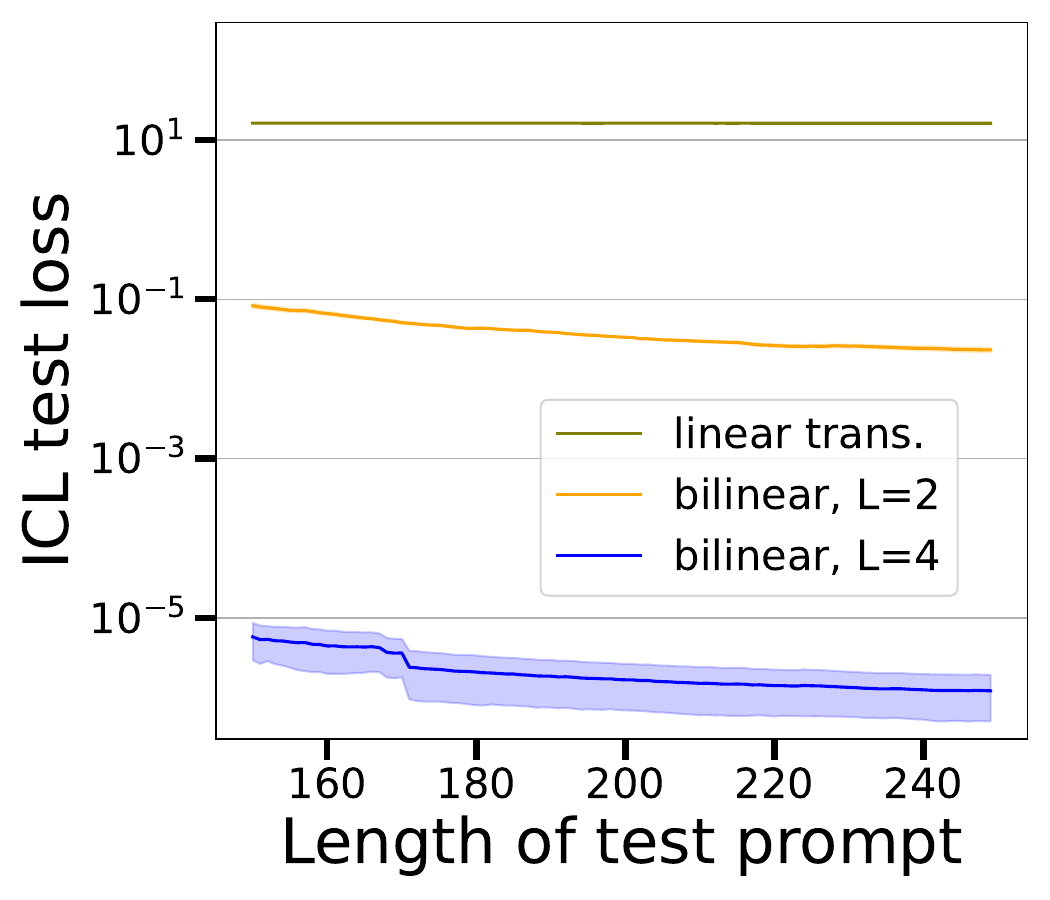}
        \caption{$d=3$}
    \end{subfigure}
    \begin{subfigure}{0.32\textwidth}
        \includegraphics[width=\textwidth]{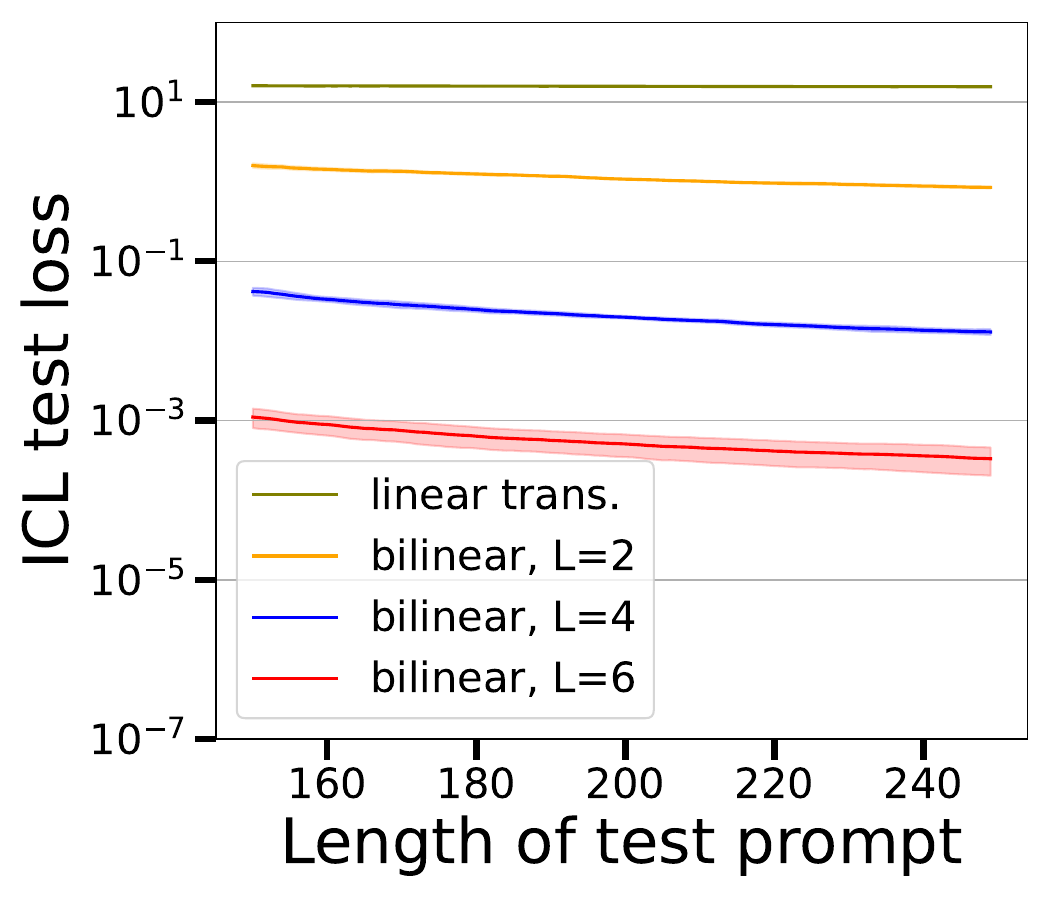}
        \caption{$d=4$}
    \end{subfigure}
    \begin{subfigure}{0.32\textwidth}
        \includegraphics[width=\textwidth]{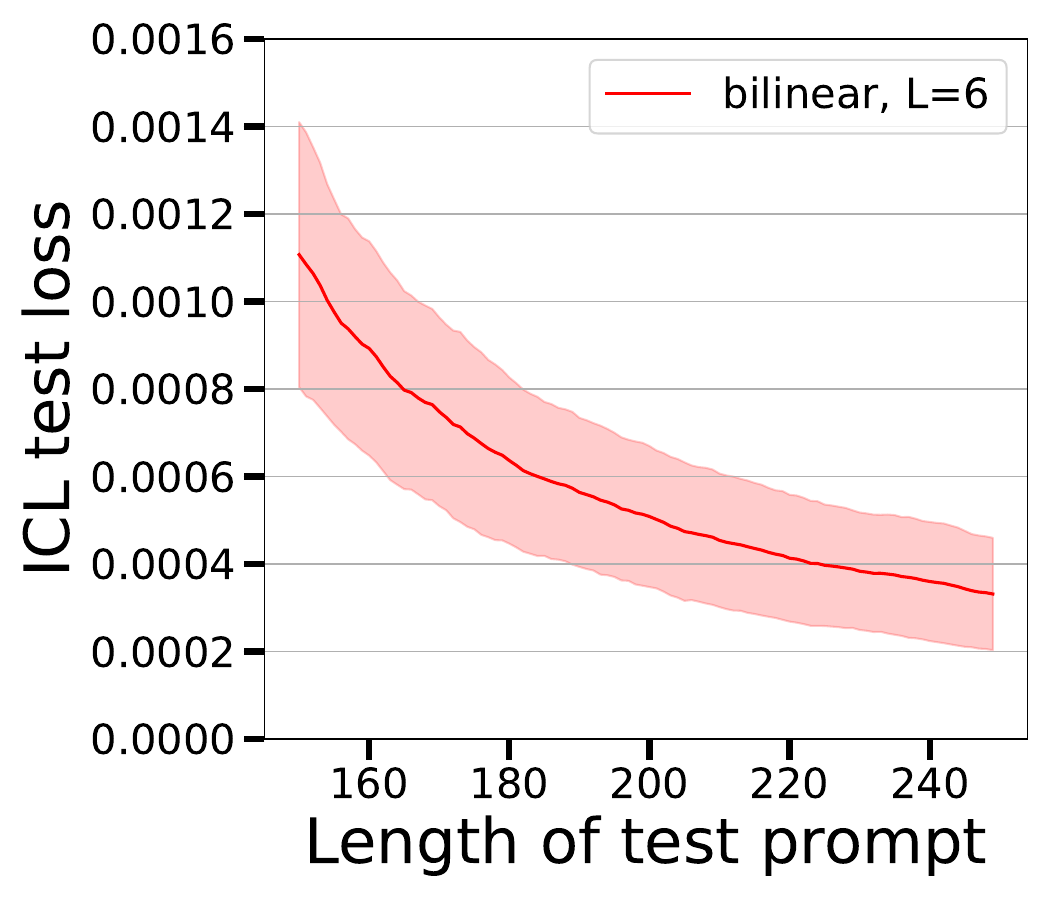}
        \caption{$d=4$, zoomed in}
    \end{subfigure}
    \caption{
        Quadratic ICL test loss (over 5 trials) of both linear and bilinear Transformer models given different test prompt lengths.
        The linear Transformer has 6 layers, and the bilinear Transformers have $2L$ layers, which implement $L$ steps of block-coordinate descent.
        In Figure (a), the quadratic function has $d=3$ variables, and Figures (b-c) have $d=4$.
        Deep bilinear models significantly outperform the linear baseline, show improved performance with increased depth ($L$), and can learn the quadratic function even with an embedding dimension ($\bar{d}=12$) insufficient to fit the number of quadratic features (15 for $d=4$).
    }
    \label{fig:block-quadratic}
\end{figure}

In Figure~\ref{fig:block-quadratic}, we plot the quadratic ICL test loss for various learned Transformer models.
The test prompts may have different lengths than the training data, but their input-label pairs $(x_i, y_i)$ are drawn i.i.d. from the same distribution.
First, in Figures~\ref{fig:block-quadratic} (a-b), the performance of the linear Transformer is significantly worse than any of the bilinear models, which is consistent with Theorem~\ref{thm:lsa-lower-bound} that linear self-attention mechanism alone cannot perform nonlinear ICL.
In Figures~\ref{fig:block-quadratic} (a-b), we also see that the performance of bilinear Transformer improves significantly with a greater number of layers.
And this trend still holds for $d=3$, when the prompt's embedding dimension is wide enough to contain all quadratic features, thus illustrating the benefits of implementing multiple gradient descent updates.
Furthermore, when $d=4$, the deep bilinear Transformers are able to overcome the limitations implied by Proposition~\ref{thm:embed-lower-bound} and successfully learn the quadratic function in-context.
Finally, in Figure~\ref{fig:block-quadratic} (c), we note that the ICL loss increases for shorter prompts, which aligns with our results in Theorem~\ref{thm:bilinear-construction}.
For additional results from these experiments, please see Appendix~\ref{sec:additional-experiment}.

\subsection{Learning higher-order polynomials in-context}
\label{sec:cubic-training}

\begin{wrapfigure}{r}{0.48\textwidth}
    \centering
    \includegraphics[width=0.4\textwidth]{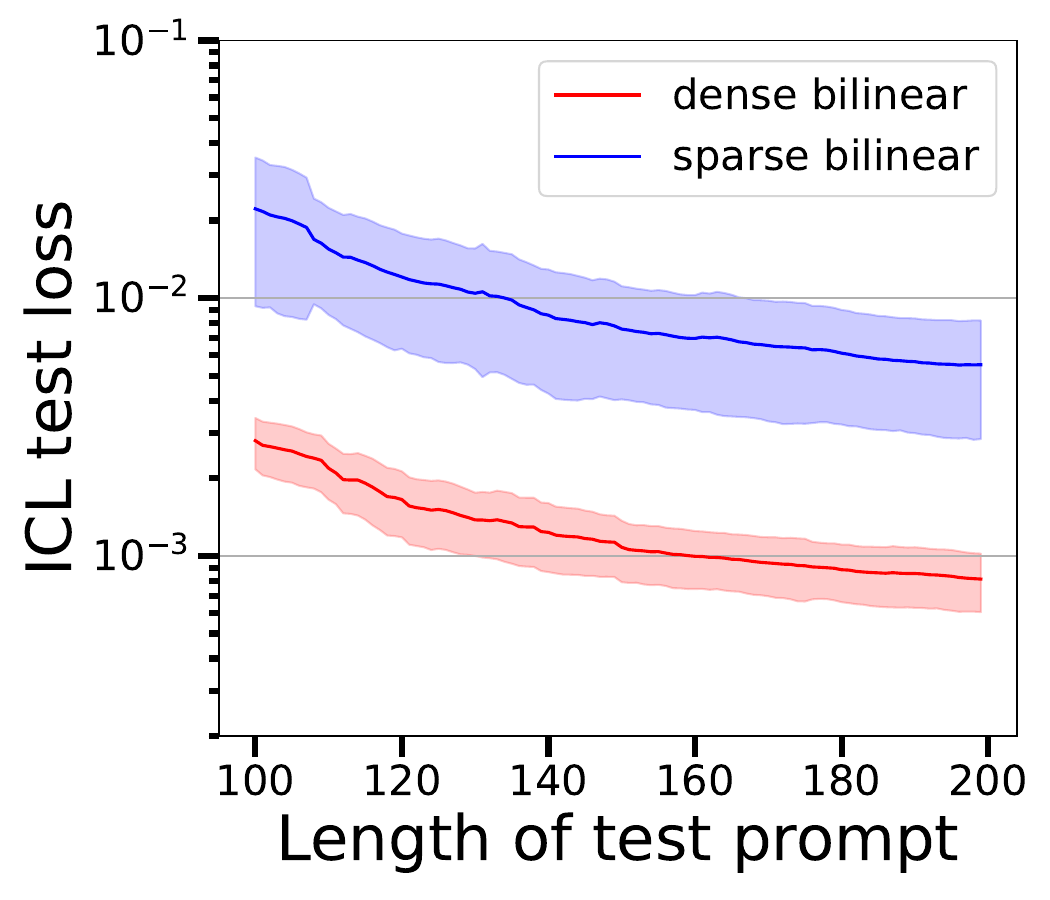}
    \caption{
        Cubic ICL test loss (over 5 trials) for bilinear Transformers with both dense \eqref{equ:dense-bilinear} and sparse \eqref{equ:sparse-bilinear} bilinear layer weights.
        The differences between the two models confirm our proposed mechanism through which bilinear Transformers can learn higher-order features in context.
    }
    \label{fig:block-cubic}
\end{wrapfigure}

In the previous section, we showed that a deep bilinear Transformer can split up the calculations of the full kernel function over multiple layers.
We extend this intuition and highlight that the same design can be used to compute a kernel generating higher-order features.
Specifically, we consider the class of degree-$p$ polynomial functions with $d$ variables:

\[ f(x; w) = \sum_{p_1 + \cdots +p_d \le p}  w_{p_1, \dots, p_d} x[1]^{p_1} \cdot x[d]^{p_d},\]
where coefficients $w$’s are i.i.d. standard normal.

Let $\bar{x}_i$ be the feature vectors generated by the first bilinear feed-forward layer.
Then, by the same argument as \eqref{equ:bilinear-kernel-formula}, the second bilinear feed-forward layer with weights $\mathbf{W}^{(2)}_0, \mathbf{W}^{(2)}_1$ computes higher-order features for each row  $j = 1, \dots, \bar{d}$:
\[    (\mathbf{W}^{(2)}_{0, j} \bar{x}_i) \cdot (\mathbf{W}^{(2)}_{1, j} \bar{x}_i) = \bar{x}_i^\top \left({{\mathbf{W}^{(2)}_{0, j}}}^\top \mathbf{W}^{(2)}_{1, j}\right) \bar{x}_i.\]
Since $\bar{x}_i$'s are quadratic in the original input $x_i$'s, the output of the second bilinear feed-forward layer can realize polynomial features up to degree 4.
Thus, a deep bilinear Transformer is capable of learning higher-order polynomials in context.
By extending this argument to Theorem~\ref{thm:block-construction}, we can derive the following result:
\begin{corollary}
    \label{thm:poly-block-construction}
    For the task of polynomial ICL, there exists a $\mathsf{TF_{bilin}}$ with $\bar{d} = 2d+1$ implementing block-coordinate descent over polynomial regression, so that for $\hat{y}^{(\ell)} := Z_{\mathsf{TF_{bilin}}}^{(2\ell)}[\bar{d}+1, n+1]$, we have $\hat{y}^{(\ell)} = f(x_{\mathsf{query}}; w^{(\ell)}),$
    where $w^{(\ell)}$ are iterates of a block-coordinate descent update \eqref{equ:block-update}.
\end{corollary}

The key distinction from Theorem~\ref{thm:block-construction} is the design of a different sequence of feature subsets that cycles through all polynomial features, which is explicitly constructed in Appendix~\ref{sec:proof-poly-block-construction}.
Furthermore, the same principle readily extends to other feed-forward layer architectures, which is discussed in Appendix~\ref{sec:extended-discussion}.

To validate this mechanism of composing features over multiple layers, we recall that the inputs $x_i$ are embedded into the prompt in the form of $[1, x_i, \mathbf{0}] \in \RR^{\bar{d}}$, where the first $d+1$ coordinates contain the original inputs and the next $\bar{d}-d-1$ coordinates can be used to store the additional nonlinear features.
If we partition the weights of a bilinear layer as
\begin{equation}
\label{equ:dense-bilinear}
\hspace{-0.5em} \mathbf{W}_i = \tbt{\mathbf{A}_i}{\mathbf{B}_i}{\mathbf{\widetilde{W}}_i}{\mathbf{C}_i}, \mathbf{\widetilde{W}_i} \in \RR^{(\bar{d}-d-1) \times (d+1)},
\end{equation}
then $\mathbf{A}_i, \mathbf{B}_i$ overwrite the space intended for the original inputs, and $\mathbf{B}_i, \mathbf{C}_i$ utilize the additional features from the previous bilinear layers to compute new features.
If we set $\mathbf{A}_i, \mathbf{B}_i, \mathbf{C}_i$ all to be zero, so that
\begin{equation}
\label{equ:sparse-bilinear}
\hspace{-0.5em} \mathbf{W}_i = \tbt{0}{0}{\mathbf{\widetilde{W}}_i}{0}, \mathbf{\widetilde{W}_i} \in \RR^{(\bar{d}-d-1) \times (d+1)},
\end{equation}
then successive bilinear layers can only use the original inputs to calculate the new features.
Therefore, if our construction accurately reflects how the bilinear Transformers perform ICL, then this sparse version of bilinear feed-forward layers would prevent us from effectively learning higher-order polynomials in-context.

To demonstrate this idea, we conduct an experiment where the target function is a cubic polynomial over 4 variables:
\begin{align*}
f(x; w) = w_0 &+ \sum_{i=1}^4 w_i x[i] + w_{12}x[1]x[2] + w_{23}x[2]x[3] \\
 &+ w_{34}x[3]x[4] + w_{123}x[1]x[2]x[3] + w_{234}x[2]x[3]x[4],
\end{align*}
where all of the coefficient $w$'s are drawn i.i.d. from $\mathcal{N}(0, 1)$.
We consider 10-layer bilinear Transformers with either the dense \eqref{equ:dense-bilinear} or the sparse \eqref{equ:sparse-bilinear} versions of the bilinear feed-forward layers.
We set the parameters to be $n=150$, $d=4$, and $\bar{d}=15$.
We note that there are 35 distinct monomials in the cubic function class of 4 variables, but we consider a cubic with just 10 terms.

As shown in Figure~\ref{fig:block-cubic}, a bilinear Transformer with dense \eqref{equ:dense-bilinear} bilinear layer weights significantly outperforms its counterpart with only sparse \eqref{equ:sparse-bilinear} bilinear layer weights.
Therefore, the bilinear Transformer with sparse bilinear layers fails to effectively learn the cubic target function in context because the sparse design inhibits the composition of higher-order features.
This observation matches up with our proposed role of the bilinear feed-forward layers in higher-order ICL.
Furthermore, the sparse bilinear Transformer can still learn quadratic features in context, which is why its test loss decreases with more in-context example.

\section{Conclusion and Future Work}
In this paper, we elucidate the essential role of feed-forward layers in enabling Transformers to perform in-context learning (ICL) of nonlinear functions.
We first establish a fundamental limitation of purely linear self-attention (LSA) models for nonlinear ICL tasks.
Using quadratic ICL as an illustrative case, we show that adding feed-forward layers enables the learning of nonlinear features  by implementing gradient-based updates in a kernel space.
However, we prove that the expressivity of any single Transformer block is fundamentally constrained by its embedding dimension and feed-forward layer size, limiting the complexity of learnable nonlinearities.
We show that deep Transformers overcome this bottleneck by distributing kernel computation across multiple feed-forward layers to learn substantially richer function classes.
Our results highlight feed-forward layers as a crucial and scalable source of nonlinear expressivity in ICL.

This work opens several promising directions for future research.
For instance, in our construction of the deep bilinear Transformer in Section~\ref{sec:efficiency}, the kernel computation can be distributed across the feed-forward layers in many ways.
It remains an open question whether there is a ``preferred'' ordering or allocation of features across the layers that optimizes learning efficiency or generalization.

\section*{Acknowledgment}
This work was supported in part by MathWorks, the MIT-IBM Watson AI Lab, the MIT-Amazon Science Hub, the MIT-Google Program for Computing Innovation, and the ONR grant \#N00014-23-1-2299.
The authors thank Amirhossein Reisizadeh, Mehrdad Ghadiri and Juan Cervino for their insightful discussions and valuable feedback.

\bibliographystyle{plainnat}

\appendix

\section{Extension to general nonlinear ICL tasks}
\label{sec:extended-discussion}
In this section, we extend our main results to arbitrary nonlinear in-context learning (ICL) tasks and to any choice of feed-forward layer architecture.
To this end, we consider $d$ variables $x_1, \dots, x_d$ drawn i.i.d. from some distribution $\mathcal{D}_X$ and let $g_1(x), \dots, g_m(x)$ be a collection of nonlinear features that are orthonormal with respect to the inner product $\inner{u}{v} = \EE_x[uv].$

We consider the general nonlinear ICL task of learning the function class
\[f(x; w) = \sum_{i=1}^m w_i g_i(x),\]
where the coefficients $w_i$'s are drawn i.i.d. from $\mathcal{N}(0, 1)$.

We first note that the lower bound established in Proposition~\ref{thm:embed-lower-bound} does not depend on the specific architecture of the feed-forward layer. Consequently, it applies directly to this generalized setting:
\begin{corollary}
    Given the general nonlinear ICL problem above,
    and let $\mathsf{TFB}$ be a Transformer block whose feed-forward layer implements a kernel mapping $x \mapsto \widetilde{x}$.
    Let $\widetilde{d}$ be the dimension of the orthogonal projection of $\widetilde{x}$ onto the span of $\{g_1(x), \dots, g_m(x)\}$.
    If $\widetilde{d} < m$, then the in-context learning loss satisfies the lower bound
    \[\mathcal{L}(\mathsf{TFB}) \ge m - \widetilde{d}.\]
\end{corollary}

This result implies that a single Transformer block --- comprising a feed-forward layer and a linear attention mechanism --- can only solve a nonlinear ICL task if both its feed-forward layer and embedding dimension are sufficiently expressive to capture all $m$ nonlinear features.

As in Section~\ref{sec:efficiency}, we can overcome this limitation by employing a deep Transformer architecture, in which the computation of nonlinear features is distributed across multiple layers.
Building on the constructions in Theorem~\ref{thm:block-construction} and Corollary~\ref{thm:poly-block-construction}, we obtain the following:
\begin{corollary}
    For the task of general nonlinear ICL above, there exists a deep Transformer $\mathsf{TF}$ with $\bar{d} = d+h+1$ implementing block-coordinate descent over kernel regression.
    Specifically, for $\ell$th layer prediction $\hat{y}^{(\ell)} := Z_{\mathsf{TF}}^{(2\ell)}[\bar{d}+1, n+1]$, we have $\hat{y}^{(\ell)} = f(x_{\mathsf{query}}; w^{(\ell)}),$
    where $w^{(\ell)}$ are iterates of a block-coordinate descent update \eqref{equ:block-update}.

    Furthermore, the sequence of feature subsets used in this block-coordinate descent update corresponds to the intermediate activations of a neural network with hidden dimension $\bar{d}$, which is obtained by stacking the feed-forward layers of $\mathsf{TF}$.
\end{corollary}
 
\section{Proof of Proposition~\ref{thm:lsa-lower-bound}}
\label{sec:proof-lsa-lower-bound}
We shall prove a more general statement.
Recall that, in \eqref{equ:prompt-form}, the prompts are written as a matrix $Z \in \RR^{(\bar{d}+1) \times (n+1)}$:
\begin{equation*}
    Z =
    \begin{bmatrix}
        1 & 1 & \cdots & 1 & 1 \\
        x_1 & x_2 & \cdots & x_n & x_{n+1} := x_{\textsf{query}} \\
        \mathbf{0} & \mathbf{0} & \cdots & \mathbf{0} & \mathbf{0} \\
        y_1 & y_2 & \cdots & y_n & 0
    \end{bmatrix},
\end{equation*}
where, for the purpose of this proposition, $x_1, \dots, x_{n+1} \overset{i.i.d.}{\sim} \mathcal{N}(0, \Sigma)$ and $y_i = f(x_i)$ for some fixed target function $f : \RR^d \to \RR$.

We note that, by defining
\begin{align*}
    u_i &= [1 \;\; x_i \;\; \mathbf{0}], \\
    g(u_i) &= g(1, x_i, \mathbf{0}) = f(x_i),
\end{align*}
the prompt \eqref{equ:prompt-form} is a special case of the form
\begin{equation}
\label{equ:prompt-form-gen}
    \zeta =
    \begin{bmatrix}
        u_1 & u_2 & \cdots & u_n & u_{n+1} := u_{\mathsf{query}}\\
        v_1 & v_2 & \cdots & v_n & 0
    \end{bmatrix},
\end{equation}
where, $u_1, \dots, u_{n+1} \overset{i.i.d.}{\sim} \mathcal{D_U}$ for a distribution $\mathcal{D_U}$ over $\RR^{\bar{d}}$ and $v_i = g(u_i)$ for some fixed target function $g : \RR^{\bar{d}} \to \RR$.

Let $\hat{v}_{\mathsf{TF}}$ be the prediction of a Transformer $\mathsf{TF}$ with prompt $\zeta$.
Then, Proposition~\ref{thm:lsa-lower-bound} is simply a special case of showing that $\hat{v}_{\mathsf{TF}}$ is no better than the best linear model over $u$.

\begin{proposition}
\label{thm:lsa-lower-bound-gen}
    Consider inputs $u_1, \dots, u_{n+1} \overset{i.i.d.}{\sim} \mathcal{D_U}$ and any fixed target function $v_i = g(u_i)$.
    The in-context learning prompts are sampled according to the form as described in \eqref{equ:prompt-form-gen}.
    Then, for any linear Transformer $\mathsf{TF_{lin}}$ (of any number of layers), its expected prediction loss satisfies:
    \[\EE_{\mathcal{D_U}}[(\hat{v}_{\mathsf{TF}_\mathsf{lin}} + v_\textsf{query})^2] \ge \min_{\beta} \EE_{\mathcal{D_U}}[(\inner{u_\textsf{query}}{\beta} + v_\textsf{query})^2].\]

\end{proposition}

\begin{proof}[Proof of Proposition~\ref{thm:lsa-lower-bound-gen}]
    Let the weights of the $\ell$th layer attention be
    \[\mathbf{Q}^{(\ell)} = \tbt{Q^{(\ell)}_{11}}{Q^{(\ell)}_{12}}{Q^{(\ell)}_{21}}{Q^{(\ell)}_{22}}, \mathbf{P}^{(\ell)} = \tbt{P^{(\ell)}_{11}}{P^{(\ell)}_{12}}{P^{(\ell)}_{21}}{P^{(\ell)}_{22}},\]
    where $Q_{11}^{(\ell)} \in \RR^{\bar{d} \times \bar{d}}, Q_{12}^{\ell} \in \RR^{\bar{d}}, Q_{21}^{(\ell)} \in \RR^{1 \times \bar{d}}$ and $Q_{22}^{(\ell)} \in \RR$, and similarly for $\mathbf{P}^{(\ell)}$.

    Consider an input matrix
    \[ \zeta^{(0)} =
    \begin{bmatrix}
        u_1 & u_2 & \cdots & u_n & u_{n+1} := u_{\mathsf{query}}\\
        v_1 & v_2 & \cdots & v_n & 0
    \end{bmatrix}
    := \tbt{U^{(0)}}{u_q^{(0)}}{{V^{(0)}}}{v_q^{(0)}},
    \]
    where the $U^{(0)} \in \RR^{\bar{d} \times n}, V^{(0)} \in \RR^{1 \times n}, u_q^{(0)} \in \RR^{\bar{d}}$ and $v_q^{(0)} \in \RR$.
    And let $\zeta^{(\ell)} = \stbt{U^{(\ell)}}{u_q^{(\ell)}}{{V^{(\ell)}}}{v_q^{(\ell)}}$ be the output of the $\ell$th layer.

    We shall use induction to show that:
    \begin{list}{{\tiny $\blacksquare$}}{\leftmargin=1em}
    \setlength{\itemsep}{-1pt}
        \item $U^{(\ell)}$ and $Y^{(\ell)}$ are independent of $u_{n+1}$.
        \item $u_q^{(\ell)}$ can be written as $A^{(\ell)} u_{n+1}$, where the matrix $A^{(\ell)}$ is independent of $u_{n+1}$.
        \item $v_q^{(\ell)}$ can be written as $\inner{u_{n+1}}{b^{(\ell)}}$, where $b^{(\ell)}$ is independent of $u_{n+1}$.
    \end{list}

    The base case when $\ell=0$ holds due to the independence of the columns.
    As for the inductive step, we compute the attention update directly. And for convenience, we shall drop superscripts $(\ell)$ and $(\ell-1)$ when they are clear from the context.
    \begin{align*}
        \zeta^{(\ell)}
        &= \zeta^{(\ell-1)} + \frac{1}{n} (\mathbf{P}^{(\ell)} \zeta^{(\ell-1)}) \tbt{I_{n}}{0}{0}{0} ({\zeta^{(\ell-1)}}^\top \mathbf{Q}^{(\ell)} {\zeta^{(\ell-1)}})   \\
        &= \zeta^{(\ell-1)} + \frac{1}{n} \left( \tbt{P_{11}^{(\ell)}}{P_{12}^{(\ell)}}{P_{21}^{(\ell)}}{P_{22}^{(\ell)}} \tbt{U^{(\ell-1)}}{u_q^{(\ell-1)}}{V^{(\ell-1)}}{v_q^{(\ell-1)}}\right) \tbt{I_{n}}{0}{0}{0} \\
        &\hspace{10em} \left(\tbt{{U^{(\ell-1)}}^\top}{{V^{(\ell-1)}}^\top}{{u_q^{(\ell-1)}}^\top}{v_q^{(\ell-1)}} \tbt{Q_{11}^{(\ell)}}{Q_{12}^{(\ell)}}{Q_{21}^{(\ell)}}{Q_{22}^{(\ell)}} \tbt{U^{(\ell-1)}}{u_q^{(\ell-1)}}{V^{(\ell-1)}}{v_q^{(\ell-1)}}\right)   \\
        &= \zeta + \frac{1}{n} \tbt{P_{11} U + P_{21} V}{P_{11} u_q + P_{12} v_q}{P_{21} U + P_{22} V}{P_{21} u_q + P_{22} y_q} \tbt{I_{n}}{0}{0}{0} \\
        &\hspace{10em} \left(\tbt{U^\top Q_{11} + V^\top Q_{21}}{U^\top Q_{12} + V^\top Q_{22}}{u_q^\top Q_{11} + v_q Q_{21}}{u_q^\top Q_{12} + v_q Q_{22}} \tbt{U}{u_q}{V}{v_q}\right)   \\
        &= \zeta + \frac{1}{n} \tbt{P_{11} U + P_{21} V}{P_{11} u_q + P_{12} v_q}{P_{21} U + P_{22} V}{P_{21} u_q + P_{22} y_q} \tbt{c_1}{c_2}{0}{0}, 
    \end{align*}
    where
    \begin{align*}
        c_1 &= U^\top Q_{11}U + V^\top Q_{21}U + U^\top Q_{12} V + V^\top Q_{22}V \\
        c_2 &= U^\top Q_{11} u_q + V^\top Q_{21} u_q + U^\top Q_{12} v_q + V^\top Q_{22} v_q.
    \end{align*}
    Therefore, we have
    \begin{align*}
        U^{(\ell)} &= U + \frac{1}{n} \left((P_{11}U + P_{21}V) (U^\top Q_{11}U + V^\top Q_{21}U + U^\top Q_{12} V + V^\top Q_{22}V)\right), \\
        Y^{(\ell)} &= Y + \frac{1}{n} \left((P_{21} U + P_{22} V) (U^\top Q_{11}U + V^\top Q_{21}U + U^\top Q_{12} V + V^\top Q_{22}V)\right), \\
        u_q^{(\ell)} &= u_q + \frac{1}{n} \left((P_{11}U + P_{21}V) (U^\top Q_{11} u_q + V^\top Q_{21} u_q + U^\top Q_{12} v_q + V^\top Q_{22} v_q)\right) \\
        &= A^{(\ell-1)} u_{n+1} + \frac{1}{n} \left((P_{11}U + P_{21}V) C_3\right) u_{n+1}, \\
        v_q^{(\ell)} &= v_q + \frac{1}{n} \left((P_{21} U + P_{22} V) (U^\top Q_{11} u_q + V^\top Q_{21} u_q + U^\top Q_{12} v_q + V^\top Q_{22} v_q)\right) \\
        &= {b^{(\ell-1)}}^\top u_{n+1} + \frac{1}{n} \left((P_{21} U + P_{22} V) C_3\right) u_{n+1},
    \end{align*}
    where
    \begin{align*}
        C_3 &= U^\top Q_{11} A^{(\ell-1)} + V^\top Q_{21} A^{(\ell-1)} + U^\top Q_{12} {b^{(\ell-1)}}^\top + V^\top Q_{22} {b^{(\ell-1)}}^\top.
    \end{align*}
    From here, it is not difficult to observe that our claim holds due to induction.

    We define $\Lambda = \EE_{\mathcal{D_U}}[u u^\top]$. Then, we compute the best linear fit as
    \begin{equation}
    \label{equ:compute-best-linear-fit}
    \begin{aligned}
        \beta^\star &:= \argmin_\beta \EE\left(\inner{\beta}{u_{n+1}} + v_{n+1}\right)^2 \\
        &= \argmin_\beta \left\{\beta^\top \Lambda \beta + 2 \beta^\top \EE_{\mathcal{D_U}}[uv] + \EE_{\mathcal{D_U}}[v^2]\right\} \\
        &= \argmin_\beta \norm{\Lambda^{1/2}\beta + \Lambda^{-1/2}\EE_{\mathcal{D_U}}[uv]}_2^2 \\
        &= -\Lambda^{-1} \EE_{\mathcal{D_U}}[uv].
    \end{aligned}
    \end{equation}
    
    Finally, we rewrite the expected error of the prediction $\hat{v}_{\mathsf{TF_{lin}}}$ as
    \begin{align*}
        \EE_{\mathcal{D_U}}\left[(\hat{v} + v_{n+1})^2\right]
        &= \EE_{\mathcal{D_U}}\left(\left(v_{n+1} + \inner{\beta^\star}{u_{n+1}}\right) - \left(\inner{\beta^\star}{u_{n+1}} - \hat{v}\right)\right)^2 \\
        &\ge \EE_{\mathcal{D_U}}\left(v_{n+1} + \inner{\beta^\star}{u_{n+1}}\right)^2 - 2 \; \EE_{\mathcal{D_U}} \left(v_{n+1} + \inner{\beta^\star}{u_{n+1}}\right)\left(\inner{\beta^\star}{u_{n+1}} - \hat{v}\right) \\
        &= \EE_{\mathcal{D_U}}\left(v_{n+1} + \inner{\beta^\star}{u_{n+1}}\right)^2 - 2 \; \EE_{\mathcal{D_U}} \, \tr \left(u_{n+1} \left(v_{n+1} + \inner{\beta^\star}{u_{n+1}}\right)\left(\beta^\star - b^{(L)}\right)^\top \right) \\
        &= \EE_{\mathcal{D_U}}\left(v_{n+1} + \inner{\beta^\star}{u_{n+1}}\right)^2 - 2 \;\tr \left(\left(\EE_{\mathcal{D_U}}[uv] + \Lambda \beta^\star\right) \left(\beta^\star - \EE_{\mathcal{D_U}} [b^{(L)}]\right)^\top \right) \\\
        &= \EE_{\mathcal{D_U}}\left(v_{n+1} + \inner{\beta^\star}{u_{n+1}}\right)^2,
    \end{align*}
    where we used our claim above on the 3rd and 4th line, and we used that fact that $\beta^\star = -\Lambda^{-1} \EE[uv]$ on the final line.
    Therefore, the population loss of the attention network's prediction cannot be better than the best linear model over the original prompt.
\end{proof}
 
\section{Proof of Theorem~\ref{thm:bilinear-construction}}
\label{sec:proof-bilinear-construction}

As we discussed in Section~\ref{sec:bilinear-construction}, a bilinear transformer block $\textsf{BTFB}$ can learn a quadratic functions in context over two steps:
\begin{list}{{\tiny $\blacksquare$}}{\leftmargin=1em}
\setlength{\itemsep}{-1pt}
    \item The bilinear feed-forward layer implements a kernel function that produces every possible quadratic monomials.
    \item The linear self-attention layer solves for a linear least-square objective on this kernel space.
\end{list}

Recall that, as described in Section~\ref{sec:in-context}, the prompt $Z \in \RR^{(\bar{d}+1) \times (n+1)}$ is specified by \eqref{equ:prompt-form}:
\begin{equation*}
    Z =
    \begin{bmatrix}
        1 & 1 & \cdots & 1 & 1 \\
        x_1 & x_2 & \cdots & x_n & x_{n+1} := x_{\textsf{query}} \\
        \mathbf{0} & \mathbf{0} & \cdots & \mathbf{0} & \mathbf{0} \\
        y_1 & y_2 & \cdots & y_n & 0
    \end{bmatrix},
\end{equation*}
where $\overline{d} = \binom{d+2}{2}$, $x_i \overset{i.i.d.}{\sim} \mathcal{N}(0, \mathbf{I}_d)$ and for coefficients $w \sim \mathcal{N}(0, \mathbf{I}_{\bar{d}})$, $y_i$'s can be written written as
\begin{equation*}
\begin{aligned}
    y_{i} 
    &= w_{00} + \sum_{j} w_{0j} x_{n+1}[j] + \sum_j w_{jj} x_{i}[j]^2 + \sum_{1 \le j < k \le d} w_{jk} x_{i}[j]  x_i[k].
\end{aligned}
\end{equation*}

\paragraph{Step 1: }
We want to find a bilinear layer whose output has the form
\begin{equation*}
    \textsf{bilin}(Z) =
    \begin{bmatrix}
        1 & 1 & \cdots & 1 & 1 \\
        \bar{x}_1 & \bar{x}_2 & \cdots & \bar{x}_n & \bar{x}_{n+1} \\
        y_1 & y_2 & \cdots & y_n & 0
    \end{bmatrix},
\end{equation*}
so that $\bar{x}$ contains every monomial in the quadratic function on $x$.

Recall that the bilinear layer is defined as
\[ \mathsf{bilin}(Z) = Z + \left(\tbt{\mathbf{W}_0}{\mathbf{0}}{\mathbf{0}}{0} Z\right) \odot \left(\tbt{\mathbf{W}_1}{\mathbf{0}}{\mathbf{0}}{0} Z\right), \]
We in fact only need parts of the weight matrices and let
\[\mathbf{W}_0 = \tbt{\mathbf{0}_{(d+1)\times(d+1)}}{\mathbf{0}}{\widetilde{\mathbf{W}}_0}{\mathbf{0}},\]
where $\widetilde{\mathbf{W}}_0$ is a $(\overline{d}-d-1) \times (d+1)$ matrix.
And similarly we define $\widetilde{\mathbf{W}}_1$.

We index rows $d+2$ to $\bar{d}$ by $\{(j, k) : 1 \le j \le k\}$.
And we define the weights by
\[ \begin{cases}
    \widetilde{\mathbf{W}}_0[(j, j), j+1] = 1, \widetilde{\mathbf{W}}_0[(j, j), 1] = -1, \widetilde{\mathbf{W}}_1[(j, j), j+1] = 1, \widetilde{\mathbf{W}}_1[(j, j), 1] = 1 & \text{if } j = k, \\
    \widetilde{\mathbf{W}}_0[(j, k), j+1] = 1, \widetilde{\mathbf{W}}_1[(j, k), k+1] = 1 & \text{otherwise}.
\end{cases} \]
Then we have
\[\begin{cases}
    \textsf{bilin}(Z)[(j, j), i] = (x_i[j] - 1)(x_i[j] + 1) = x_i[j]^2 - 1 & \text{if } j = k, \\
    \textsf{bilin}(Z)[(j, k), i] = x_i[j]x_i[k] & \text{otherwise}.
\end{cases}\]
Combining with the residual connection, we have
\[\bar{x}_i = (1, x_i[1], \dots, x_i[d], x_i[1]^2-1, \dots, x_i[1]x_i[d], x_i[2]^2-1, \dots, x_i[2]x_i[d], \dots, x_i[d]^2-1),\]
which matches the desirable kernel in \eqref{equ:kernel}.

\paragraph{Step 2:}
We note that in the output of the bilinear layer, we can write $y_i$'s as a linear function in $\bar{x}_i$'s:
\begin{equation}
\label{equ:restruct-quadratic-2}
\begin{aligned}
    y_{i} 
    &= w_{00} + \sum_{j} w_{0j} x_{n+1}[j] + \sum_j w_{jj} x_{i}[j]^2 + \sum_{1 \le j < k \le d} w_{jk} x_{i}[j]  x_i[k] \\
    &= \left(w_{00} + \sum_{j=1}^d w_{jj}\right) + \sum_{j} w_{0i} x_{n+1}[j] + \sum_j w_{jj} (x_{i}[j]^2 - 1)^2 + \sum_{1 \le j < k \le d} w_{jk} x_{i}[j]  x_i[k] \\
    &=: \sum_{0 \le j \le k \le d} \bar{w}_{jk} \bar{x}_i[j, k].\\
\end{aligned}
\end{equation}

Therefore, we can use the linear self-attention layer to implement a linear least-square estimator.
To compute the in-context learning loss, we leverage the following lemma.

\begin{lemma}
\label{thm:attn-loss}
    Consider an ICL prompt
    \begin{equation}
    \label{equ:prompt-form-gen-2}
        \zeta =
        \begin{bmatrix}
            u_1 & u_2 & \cdots & u_n & u_{n+1} := u_{\mathsf{query}}\\
            v_1 & v_2 & \cdots & v_n & 0
        \end{bmatrix}
        =: \tbt{U}{u_q}{V}{0},
    \end{equation}
    where, $(u_1, v_1), \dots, (u_{n+1}, v_{n+1}) \overset{i.i.d.}{\sim} \mathcal{D}$ for a distribution $\mathcal{D}$ over $\RR^{\bar{d}} \times \RR$.
    Let 
    \[ \Lambda = \EE_{\mathcal{D}}[u u^\top], \Phi = \mathrm{Cov}_{\mathcal{D}}[uv], \xi = \EE_{\mathcal{D}}[uv]. \]
    Then, for a single-layer linear self-attention with weights
    \[\mathbf{P} = \tbt{0_{d \times d}}{0}{0}{1}, \mathbf{Q} = \tbt{\Gamma}{0}{0}{0},\]
    the error of the prediction $\hat{v} = \mathsf{attn}(\zeta)[\bar{d}+1, n+1]$ satisfies
    \[\EE_\mathcal{D}\left[(\hat{v} + v_{n+1})^2\right] = \min_\beta \EE\left(\inner{\beta}{u_{n+1}} + v_{n+1}\right)^2 + \frac{1}{n} \tr\left(\Gamma \Lambda \Gamma^\top \Phi \right) + \xi^\top (\Gamma^\top + {\Lambda^{-1}})^\top \Lambda (\Gamma^\top + \Lambda^{-1}) \xi.\]
\end{lemma}

We apply this lemma with $u_i = \bar{x}_i$, $v_i = y_i$ and the distribution $\mathcal{D}$ corresponds to the distribution of $x_i$'s and fixed coefficients $w$.
Then,
\begin{align*} 
    \mathcal{L}(\textsf{BTFB}) 
    &= \EE_w\left[\EE_x\left[(\hat{y} + y_{n+1})^2 \mid w\right]\right] \\
    &= \EE_w\left[\min_\beta \EE\left(\inner{\beta}{\bar{x}_{n+1}} + y_{n+1}\right)^2 + \frac{1}{n} \tr\left(\Gamma \Lambda \Gamma^\top \Phi \right) + \xi^\top (\Gamma^\top + {\Lambda^{-1}})^\top \Lambda (\Gamma^\top + \Lambda^{-1}) \xi\right] \\
    &= \EE_w\left[\frac{1}{n} \tr\left(\Gamma \Lambda \Gamma^\top \Phi \right) + \xi^\top (\Gamma^\top + {\Lambda^{-1}})^\top \Lambda (\Gamma^\top + \Lambda^{-1}) \xi\right],
\end{align*}
where
\[ \Lambda = \EE_x[\bar{x}\bar{x}^\top], \Phi = \mathrm{Cov}_x[\bar{x}y \mid w], \xi = \EE_x[\bar{x}y \mid w], \]
and the last line follows because from \eqref{equ:restruct-quadratic-2}, we know that $y_{n+1}$ is linear in $\bar{x}_{n+1}$.

Note that, if we choose $\Gamma = -\Lambda^{-1}$, then
\[\mathcal{L}(\textsf{BTFB}) = \frac{1}{n} \EE_w\left[\tr\left(\Lambda^{-1} \Phi \right)\right].\]

It remains to compute $\EE_w\left[\tr\left(\Lambda^{-1} \Phi \right)\right] = \tr\left(\Lambda^{-1} \EE_w[\Phi] \right)$.
First, due to Lemma~\ref{thm:quadratic-basis}, $\Lambda = \EE[\bar{x}\bar{x}^\top]$ is diagonal with nonzero entries equal to either 1 or 2.
Next, we compute $\xi$, where we apply the expression for $y$ as written in \eqref{equ:restruct-quadratic-2},
\begin{align*}
    \xi = \EE_x[\bar{x}y] 
    = \EE_x[\bar{x}\inner{\bar{x}}{\bar{w}}] 
    = \left[\EE_x[\bar{x}[k]\inner{\bar{x}}{\bar{w}}] \vphantom{x^2}\right]_{k=1}^{\bar{d}}
    = \bigg[\EE_x[\bar{x}[k]^2]\bar{w}[k]\bigg]_{k=1}^{\bar{d}}
    = \Lambda \bar{w}.
\end{align*}
Taking the expectation over $w$ yields that
\[\EE_w[\Phi] = \Lambda \EE_w[\bar{w}\bar{w}^\top] \Lambda.\]
And since $w \sim \mathcal{N}(0, \mathbf{I}_{\bar{d}})$, the covariance matrix $\EE_w[\bar{w}\bar{w}^\top]$ after some permutation of the coordinates can be written as:
\[\begin{bmatrix}
    d+1 & \mathbf{1}_{1 \times d} & \mathbf{0} \\
    \mathbf{1}_{d \times 1} & \mathbf{I}_{d} & \mathbf{0} \\
    \mathbf{0} & \mathbf{0} & \mathbf{I}_{\bar{d}-d-1}
\end{bmatrix},\]
where the top-left $(d+1) \times (d+1)$ submatrix correspond to the coefficients $\bar{w}_{00}, \bar{w}_{11}, \dots, \bar{w}_{dd}$.
And under the same permutation, $\Lambda = \mathrm{diag}(1, 2\mathbf{I}_d, \mathbf{I}_{\bar{d}-d-1})$.

Therefore,
\[\mathcal{L}(\textsf{BTFB}) = \frac{1}{n} \EE_w\left[\tr\left(\Lambda^{-1} \Phi \right)\right] = \frac{1}{n} \left(d+1 + d/2 + \binom{d+2}{2}-d-1\right) = \frac{(d+2)(d+1) + d}{2n}.\]

\subsection{Proof of Lemma~\ref{thm:attn-loss}}
Recall that from our computation in \eqref{equ:compute-best-linear-fit}
\[\beta^\star := \argmin_\beta \EE\left(\inner{\beta}{u_{n+1}} + v_{n+1}\right)^2 = -\Lambda^{-1} \xi.\]
Let $\hat\beta = -\Gamma^\top \xi$.
We rewrite the expected error as
\begin{align*}
    \EE\left[(\hat{v} + v_{n+1})^2\right]
    &= \EE\left(\left(v_{n+1} + \inner{\beta^\star}{u_{n+1}}\right) + \left(\inner{\hat\beta}{u_{n+1}} + \hat{v}\right) - \left(\inner{\beta^\star}{u_{n+1}} + \inner{\hat\beta}{u_{n+1}}\right)\right)^2 \\
    &= \underbrace{\EE\left(v_{n+1} + \inner{\beta^\star}{u_{n+1}}\right)^2}_{\textcircled{1}} + \underbrace{\EE\left(\inner{\hat\beta}{u_{n+1}} + \hat{v}\right)^2}_{\textcircled{2}} + \underbrace{\EE\left(\inner{\beta^\star}{u_{n+1}} + \inner{\hat\beta}{u_{n+1}}\right)^2}_{\textcircled{3}} \\
    &\quad - \underbrace{2 \; \EE\left[\left(v_{n+1} + \inner{\beta^\star}{u_{n+1}}\right) \left(\inner{\beta^\star}{u_{n+1}} + \inner{\hat\beta}{u_{n+1}}\right)\right]}_{\textcircled{4}} \\
    &\quad - \underbrace{2 \; \EE\left[\left(\inner{\hat\beta}{u_{n+1}} + \hat{v}\right) \left(\inner{\beta^\star}{u_{n+1}} + \inner{\hat\beta}{u_{n+1}}\right)\right]}_{\textcircled{5}} \\
    &\quad + \underbrace{2 \; \EE\left[\left(v_{n+1} + \inner{\beta^\star}{u_{n+1}}\right) \left(\inner{\hat\beta}{u_{n+1}} + \hat{v}\right)\right]}_{\textcircled{6}}
\end{align*}

To further simplify these terms, we must first derive an explicit form for $\hat{v}$.
To this end, we directly expand the attention layer:
\begin{align*}
    \zeta^+
    &= \zeta + \frac{1}{n} (\mathbf{P} \zeta) \tbt{I_{n}}{0}{0}{0} ({\zeta}^\top \mathbf{Q} {\zeta})   \\
    &= \zeta + \frac{1}{n} \tbt{0_{d\times d}}{0}{0}{1} \tbt{U}{u_q}{V}{0} \tbt{I_{n}}{0}{0}{0} \left(\tbt{U^\top}{V^\top}{u_q^\top}{0} \tbt{\Gamma}{0}{0}{0} \tbt{U}{u_q}{V}{0}\right) \\
    &= \zeta + \frac{1}{n} \tbt{0}{0}{V}{0} \tbt{I_n}{0}{0}{0} \tbt{U^\top}{V^\top}{u_q^\top}{0} \tbt{\Gamma U}{\Gamma u_q}{0}{0} \\
    &= \zeta + \frac{1}{n} \tbt{0}{0}{V}{0} \tbt{U^\top \Gamma U}{U^\top \Gamma u_q}{0}{0} \\
    &= \tbt{U}{u_q}{V + \frac{1}{n} V U^\top \Gamma U}{\frac{1}{n} V U^\top \Gamma u_q}.
\end{align*}
Therefore, we have that
\[\hat{v} = \frac{1}{n} V U^\top \Gamma u_q = \frac{1}{n} \inner{\sum_{i=1}^n v_i u_i}{\Gamma u_{n+1}}.\]
Then, for any fixed $u$, we have
\begin{align*}
    \EE\left[\inner{\hat\beta}{u_{n+1}} + \hat{v} \mid u_{n+1} = u\right]
    = \EE\left[\frac{1}{n} \sum_{i=1}^n \Gamma^\top v_i u_i + \hat\beta\right]^\top u = 0.
\end{align*}
Therefore, under the tower property of expectation, terms \textcircled{5} and \textcircled{6} are 0.

As for term \textcircled{4}, we have
\begin{align*}
    \frac{1}{2} \textcircled{4}
    &= \EE\left[\tr\left(u_{n+1} (v_{n+1} + \inner{\beta^\star}{u_{n+1}}) (\beta^\star + \hat\beta)^\top\right)\right] \\
    &= \tr\left((\xi + \Lambda \beta^\star)(\beta^\star+\hat\beta)^\top\right) = 0.
\end{align*}

For term \textcircled{2}, we have
\begin{align*}
    \textcircled{2}
    &= \EE \left[\left(\frac{1}{n} \sum_{i=1}^n \Gamma^\top v_i u_i + \hat\beta\right)^\top u_{n+1} {u_{n+1}}^\top \left(\frac{1}{n} \sum_{i=1}^n \Gamma^\top v_i u_i + \hat\beta\right) \right] \\
    &= \EE \left[\left(\frac{1}{n} \sum_{i=1}^n v_i u_i - \xi \right)^\top \Gamma u_{n+1} {u_{n+1}}^\top \Gamma^\top \left(\frac{1}{n} \sum_{i=1}^n v_i u_i - \xi \right) \right] \\
    &= \frac{1}{n} \tr\left(\Gamma \Lambda \Gamma^\top \Phi \right)
\end{align*}

Latsly, for term \textcircled{3}, we have
\begin{align*}
    \textcircled{3}
    &= \EE \left[(\beta^\star + \hat\beta)^\top u_{n+1} {u_{n+1}}^\top (\beta^\star + \hat\beta)\right] \\
    &= \xi^\top (\Gamma^\top + {\Lambda^{-1}})^\top \Lambda (\Gamma^\top + \Lambda^{-1}) \xi.
\end{align*}

Adding up these terms yields the desirable equation.

\subsection{Verification of Proposition~\ref{thm:lsa-lower-bound}}
\label{sec:verification}

\begin{figure}[!htb]
    \centering
    \includegraphics[width=0.4\textwidth]{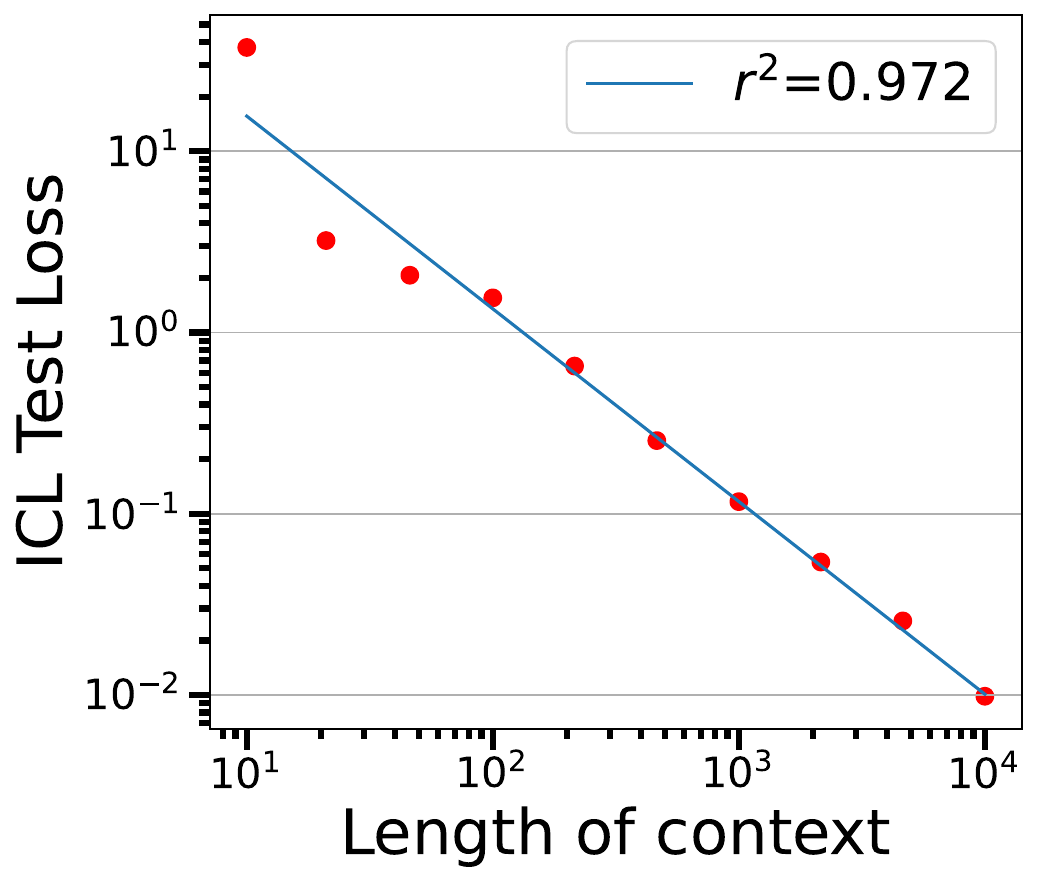}
    \caption{
        Quadratic ICL loss of the bilinear Transformer block construction associated with Proposition~\ref{thm:bilinear-construction}.
    }
    \label{fig:verification}
\end{figure}

To verify Proposition~\ref{thm:lsa-lower-bound}, we apply the bilinear Transformer block construction associated with this result and compute its quadratic ICL loss at various context lengths $n$.
Proposition~\ref{thm:bilinear-construction} states that the ICL loss should scale inversely with $1/n$, which corresponds to a straight line on the log-log scale.
This theoretical prediction is consistent with our empirical findings, shown in Figure~\ref{fig:verification}, where the linear fit on the log-log scale yields an $r^2$ value of 0.97.

It’s important to note that our construction provides only an upper bound on the optimal ICL loss, and much smaller context lengths $n$ may suffice to achieve satisfactory performance.
Additionally, multiple copies of the bilinear Transformer block can be stacked to improve ICL performance, enabling the Transformer to mimic multiple steps of gradient descent on quadratic kernel regression.
 
\section{Proof of Lemma~\ref{thm:quadratic-basis}}
\label{sec:proof-quadratic-basis}

Recall that, in \eqref{equ:quadratic-basis}, we consider a collection of random variable
\begin{equation*}
    \left\{v_{ij} \vphantom{x^2}\right\}_{0 \le i \le j \le d} :=
    \begin{cases}
            1 & \text{if  } i = j = 0, \\
            x[j] & \text{if  } i= 0, j \in \{1, \dots, d\}, \\
            \frac{1}{\sqrt{2}} (x[i]^2-1) & \text{if  } i = j \in \{1, \dots, d\}, \\
            x[i]x[j] & \text{if  } 1 \le i < j \le d,
    \end{cases}
\end{equation*}
where $x \sim \mathcal{N}(0, \mathbf{I}_d)$.
We also define inner product of random variables as $\inner{u}{v} := \EE_x[uv]$.

We check that they are unit norm:
\begin{list}{{\tiny $\blacksquare$}}{\leftmargin=1em}
\setlength{\itemsep}{-1pt}
    \item $\EE[v_{00}^2] = \EE[1^2] = 1$.
    \item $\EE[v_{0j}^2] = \EE[x[j]^2] = 1$.
    \item $\EE[v_{ii}^2] = \EE[\frac{1}{2}(x[i]^2 - 1)^2)] = \frac{1}{2}(\EE[x[i]^4] - 2\EE[x[i]^2] + 1) = 1.$
    \item $\EE[v_{ij}^2] = \EE[x[i]^2 x[j]^2] = \EE[x[i]^2] \EE[x[j]^2] = 1$.
\end{list}

Then, we check that they are orthogonal:
\begin{list}{{\tiny $\blacksquare$}}{\leftmargin=1em}
\setlength{\itemsep}{-1pt}
    \item $\EE[v_{00}v_{0j}] = \EE[x[j]] = 0$.
    \item $\EE[v_{00}v_{ii}] = \EE[\frac{1}{\sqrt{2}} (x[i]^2-1)] = \frac{1}{\sqrt{2}}(\EE[x[i]^2] - 1) = 0$.
    \item $\EE[v_{00}v_{ij}] = \EE[x[i]x[j]] = \EE[x[i]] \EE[x[j]] = 0$.
    \item $\EE[v_{0j}v_{ii}] = \EE[x[j]\frac{1}{\sqrt{2}} (x[i]^2-1)] = \frac{1}{\sqrt{2}} (\EE[x[j]x[i]^2] - \EE[x[j]]) = \frac{1}{\sqrt{2}} \EE[x[j]x[i]^2]$ and note that
    \[\EE[x[j]x[i]^2] =
    \begin{cases}
        \EE[x[j]] \EE[x[i]]^2 = 0 & \text{if } i \neq j,\\
        \EE[x[i]^3] = 0 & \text{if } i = j.
    \end{cases}\]
    \item We have
    \[\EE[v_{0k} v_{ij} = \EE[x[k] \cdot x[i]x[j]] = 
    \begin{cases} 
        \EE[x[i]^2]\EE[x[j]] = 0& \text{if } k = i, \\
        \EE[x[i]]\EE[x[j]^2] = 0 & \text{if } k = j, \\
        \EE[x[k]] \EE[x[i]] \EE[x[j]] = 0& \text{otherwise.}
    \end{cases}\] 
    \item $\EE[v_{kk} v_{ij}] = \EE[\frac{1}{\sqrt{2}} (x[k]^2-1) \cdot x[i][j]] = \frac{1}{\sqrt{2}} (\EE[x[k]^2x[i]x[j]] - \EE[x[i]x[j]]) = \frac{1}{\sqrt{2}} \EE[x[k]^2x[i]x[j]]$, and note that
    \[\EE[x[k]^2 \cdot x[i]x[j]] = 
    \begin{cases} 
        \EE[x[i]^3]\EE[x[j]] = 0& \text{if } k = i, \\
        \EE[x[i]]\EE[x[j]^3] = 0 & \text{if } k = j, \\
        \EE[x[k]^2] \EE[x[i]] \EE[x[j]] = 0& \text{otherwise.}
    \end{cases}\] 
\end{list}

We conclude that $\{v_{ij}\}_{0 \le i \le j \le d}$ form an orthonormal set.

\section{Proof of Proposition~\ref{thm:embed-lower-bound}}
\label{sec:proof-embed-lower-bound}

Recall that we consider inputs $x_i \overset{i.i.d.}{\sim} \mathcal{N}(0, I_d)$.
Let the output of the bilinear layer of $\mathsf{BTFB}$ be 
\begin{equation*}
    \bar{Z} =
    \begin{bmatrix}
        \widetilde{x}_1 & \widetilde{x}_2 & \cdots & \widetilde{x}_n & \widetilde{x}_{n+1} \\
        y_1 & y_2 & \cdots & y_n & 0
    \end{bmatrix},
\end{equation*}
where $\widetilde{x}_i\in \RR^{\bar{d}}$.
Given Proposition~\ref{thm:lsa-lower-bound-gen}, we have that
\[\mathcal{L}(\mathsf{BTFB}) = \EE_w \left[ \EE_x[(\hat{y} + y_{n+1})^2] \right] \ge \EE_w\left[\min_\beta \EE_x\left(y_{n+1} + \inner{\beta}{\widetilde{x}_{n+1}}\right)^2\right].\]

We consider the orthonormal set $\{v_{ij}\}_{0 \le i \le j \le d}$ as defined in \eqref{equ:quadratic-basis}. 
Next, define the projection operator $P_1$ on a random variable $z$ as:
\[P_1(z) = \sum_{0 \le i \le j \le d} \EE[z v_{ij}] v_{ij}.\]
Note that 
\[\EE[(z - P_1(z)) v_{ij}] = \EE[z v_{ij}] - \sum_{0 \le i' \le j' \le d} \EE[z v_{i'j'}] \EE[v_{ij} v_{i'j'}] = 0.\]
So, we have $\EE[(z - P_1(z)) P_1(z)] = 0.$
Recall that, for fixed coefficients $w$, we have
\begin{align*}
    y_{n+1} 
    &= w_{00} + \sum_{i} w_{0i} x_{n+1}[i] + \sum_i w_{ii} x_{n+1}[i]^2 + \sum_{1 \le i < j \le d} w_{ij} x_{n+1}[i]  x_{n+1}[j] \\
    &= \left(w_{00} + \sum_{i=1}^d w_{ii}\right) + \sum_{i} w_{0i} x_{n+1}[i] + \sum_i w_{ii} (x_{n+1}[i]^2 - 1) + \sum_{1 \le i < j \le d} w_{ij} x_{n+1}[i]  x_{n+1}[j]
\end{align*}
Then, by applying Lemma~\ref{thm:quadratic-basis}, it is not difficult to see that 
\begin{align*} 
    P_1(y_{n+1}) 
    &= \left(w_{00} + \sum_{i=1}^d w_{ii}\right) + \sum_{i=1}^d w_{0i} v_{0i} + \sum_{i=1}^d \sqrt{2} w_{ii} v_{ii} + \sum_{1 \le i < j \le d} w_{ij} v_{ij} \\
    &=: \bar{w}_{00} + \sum_{i} \bar{w}_{0i} v_{0i} + \sum_i \bar{w}_{ii} v_{ii} + \sum_{1 \le i < j \le d} \bar{w}_{ij} v_{ij}.\\
\end{align*}
Therefore, for any fixed $\beta \in \RR^{\bar{d}}$, we have that:
\begin{align*}
    \EE_x\left(y_{n+1} + \inner{\beta}{\widetilde{x}_{n+1}}\right)^2 
    &\ge \EE_x\left(P_1(y_{n+1}) + P_1\left(\inner{\beta}{\widetilde{x}_{n+1}}\right)\right)^2 \\
    &= \EE_x\left(\sum_{0 \le i \le j \le d}\bar{w}_{ij} v_{ij} +\sum_{k=1}^{\bar{d}} \beta_k \sum_{0 \le i \le j \le d} \EE\left[v_{ij}\widetilde{x}_{n+1}[k]\right]v_{ij}\right)^2 \\
    &= \EE_x \sum_{0 \le i < j \le d} \left(\left(\bar{w}_{ij} + \sum_{k=1}^{\bar{d}} \beta_k \EE\left[v_{ij}\widetilde{x}_{n+1}[k]\right] \right)v_{ij}\right)^2 \\
    &= \sum_{0 \le i < j \le d} \left(\bar{w}_{ij} + \sum_{k=1}^{\bar{d}} \beta_k \EE\left[v_{ij}\widetilde{x}_{n+1}[k]\right] \right)^2.
\end{align*}

Define a $\binom{d+2}{2}$-dimensional real vector space indexed by $\{(i, j) : 0 \le i \le j \le d\}$.
Define vectors $\vec{w}$ and $\{u_k\}_{k=1}^{\bar{d}}$ on this vector space so that their are entries are:
\[ \vec{w} = \left[\bar{w}_{ij} \vphantom{w^2}\right]_{0 \le i \le j \le d}, u_k = \left[\EE[v_{ij} \widetilde{x}_{n+1}[k]] \vphantom{w^2}\right]_{0 \le i \le j \le d}.\]
Note that, by the theorem statement, the span of $u_k$'s have dimension $\tilde{d} < \binom{d+2}{2}$.

Let $P_2$ be the orthogonal projection onto the the span of $u_k$'s.
Then, for any fixed $w$, we have
\begin{align*} 
    \min_\beta \EE_x\left(y_{n+1} + \inner{\beta}{\widetilde{x}_{n+1}}\right)^2 
    &\ge \min_\beta \sum_{0 \le i \le j \le d} \left(\bar{w}_{ij} + \sum_k \beta_k \EE[v_{ij}\widetilde{x}_{n+1}[k]] \right)^2 \\
    &= \min_\beta \sum_{0 \le i \le j \le d} \left(\bar{w}_{ij} + \sum_k \beta_k \EE[v_{ij}\widetilde{x}_{n+1}[k]] \right)^2 \\
    &= \norm{\vec{w} - P_2(\vec{w})}^2.
\end{align*}

Now we take the expectation over $w$ and let $\mathrm{Cov}[\vec{w}] = \Sigma$.
\begin{align*}
    \EE_w\left[\min_\beta \EE_x\left(y_{n+1} + \inner{\beta}{\widetilde{x}_{n+1}}\right)^2\right]
    &\ge \EE_w \left[\norm{\vec{w} - P_2(\vec{w})}^2 \right] \\
    &= \EE\left[\tr\left((I-P_2) \vec{w}\vec{w}^\top\right)\right]\\
    &= \tr((I - P_2)\Sigma),
\end{align*}
where we used the fact that $P_2$ is idempotent.
Since $\Sigma$ is positive semi-definite, we have
\[\tr((I - P_2)\Sigma) = \tr\left(\Sigma^{1/2}(I - P_2)(I - P_2)\Sigma^{1/2}\right) = \norm{(I - P_2)\Sigma^{1/2}}_F^2.\]
And due to the matrix norm identity $\norm{AB}_F \le \norm{A}_F \norm{B}_{\mathrm{op}}$, we have
\[\norm{I - P_2}_F \le \norm{(I-P_2)\Sigma^{1/2}}_F \cdot \norm{\Sigma^{-1/2}}_{\rm op}.\]

Since $w \sim \mathcal{N}(0, \mathbf{I}_{\bar{d}})$, after some permutation of the coordinates, the covariance matrix $\Sigma = \EE_w[\bar{w}\bar{w}^\top]$  can be written as:
\[\begin{bmatrix}
    d+1 & \mathbf{1}_{1 \times d} & \mathbf{0} \\
    \mathbf{1}_{d \times 1} & 2 \mathbf{I}_{d} & \mathbf{0} \\
    \mathbf{0} & \mathbf{0} & \mathbf{I}_{\bar{d}-d-1}
\end{bmatrix},\]
where the top-left $(d+1) \times (d+1)$ submatrix corresponds to the coefficients $\bar{w}_{00}, \bar{w}_{11}, \dots, \bar{w}_{dd}$.
The eigenvalues of this top-left submatrix are $d+2, \, 1, \, 2, \, \dots, \, 2$ with corresponding eigenvectors
\[[d, 1, \dots 1], [-1, 1, \dots 1], [0, -1, 1, 0, \dots], [0, -1, 0, 1, 0, \dots], \dots, [0, -1, 0, \dots, 1].\]

Therefore, the minimum eigenvalue of $\Sigma$ is 1.
It follows that
\[\tr((I - P_2)\Sigma) = \norm{(I - P_2)\Sigma^{1/2}}_F^2 \ge \norm{I - P_2}_F^2 = \binom{d+2}{2} - \widetilde{d},\]
where the last equality follows from the fact that $I-P_2$ is also an orthogonal projection matrix.
We conclude that
\[\mathcal{L}(\mathsf{BTFB})  \ge \binom{d+2}{2} - \widetilde{d}.\]

\section{Proof of Theorem~\ref{thm:approx-opt}}
\label{sec:proof-approx-opt}

Recall that we are given the bilinear layer corresponding to the kernel function \eqref{equ:kernel}, where
\[\bar{x}_i = (1, x_i[1], \dots, x_i[d], x_i[1]^2-1, \dots, x_i[1]x_i[d], x_i[2]^2-1, \dots, x_i[2]x_i[d], \dots, x_i[d]^2-1),\]
with $x_i \overset{\mathrm{i.i.d.}}{\sim} \mathcal{N}(0, \mathbf{I}_d)$.
The prompt given to the self-attention layer takes the form
\begin{equation*}
    Z =
    \begin{bmatrix}
        \bar{x}_1 & \bar{x}_2 & \cdots & \bar{x}_1 & \bar{x}_{n+1} := \bar{x}_{\textsf{query}} \\
        y_1 & y_2 & \cdots & y_n & 0
    \end{bmatrix}
    =: \tbt{X}{x_q}{Y}{0}.
\end{equation*}
Let $\bar{d} = \binom{d+2}{2}$.
And for $w \sim \mathcal{N}(0, \mathbf{I}_{\bar{d}})$, we have
\begin{equation}
\label{equ:restruct-quadratic}
\begin{aligned}
    y_{i} 
    &= w_{00} + \sum_{j} w_{0j} x_{n+1}[j] + \sum_j w_{jj} x_{i}[j]^2 + \sum_{1 \le j < k \le d} w_{jk} x_{i}[j]  x_i[k] \\
    &= \left(w_{00} + \sum_{j=1}^d w_{jj}\right) + \sum_{j} w_{0i} x_{n+1}[j] + \sum_j w_{jj} (x_{i}[j]^2 - 1)^2 + \sum_{1 \le j < k \le d} w_{jk} x_{i}[j]  x_i[k] \\
    &=: \sum_{0 \le j \le k \le d} \bar{w}_{jk} \bar{x}_i[j, k].\\
\end{aligned}
\end{equation}
For convenience, we sometimes use a single index so that $y_i = \sum_{\ell=1}^{\bar{d}} \bar{w}_{\ell} \bar{x}_\ell$.

There are two parts to this proof:
\begin{list}{{\tiny $\blacksquare$}}{\leftmargin=1em}
\setlength{\itemsep}{-1pt}
    \item Show that the optimal linear self-attention weights must have the sparsity pattern
        \[\mathbf{P} = \tbt{0_{d \times d}}{0}{0}{1}, \mathbf{Q} = \tbt{\Gamma}{0}{0}{0}.\]
    \item Show that any choice of $\Gamma$ that deviates too much from $-\EE[\bar{x}\bar{x}^\top]^{-1}$ must have a greater ICL loss than our upper bound in Theorem~\ref{thm:bilinear-construction}.
\end{list}

\paragraph{Part 1:}
Let us write the attention weights as
\[\mathbf{Q} = \tbt{Q_{11}}{Q_{12}}{Q_{21}}{Q_{22}}, \mathbf{P} = \tbt{P_{11}}{P_{12}}{P_{21}}{P_{22}},\]
where $Q_{11} \in \RR^{d \times d}, Q_{12} \in \RR^{d}, Q_{21} \in \RR^{1 \times d}$ and $Q_{22} \in \RR$, and similarly for $\mathbf{P}$.

Then, we expand the definition of the self-attention layer:
{\small
\begin{align*}
    \mathsf{attn}(Z)
    &= Z + \frac{1}{n} (\mathbf{P} Z) \tbt{I_{n}}{0}{0}{0} (Z^\top \mathbf{Q} Z)   \\
    &= Z + \frac{1}{n} \left( \tbt{P_{11}}{P_{12}}{P_{21}}{P_{22}} \tbt{X}{x_q}{Y}{0}\right) \tbt{I_{n}}{0}{0}{0} \left(\tbt{X^\top}{Y^\top}{x_q^\top}{0} \tbt{Q_{11}}{Q_{12}}{Q_{21}}{Q_{22}} \tbt{X}{x_q}{Y}{0}\right)   \\
    &= Z + \frac{1}{n} \tbt{P_{11} X + P_{21} Y}{P_{11} x_q}{P_{21} X + P_{22} Y}{P_{21} x_q} \tbt{I_{n}}{0}{0}{0} \left(\tbt{X^\top Q_{11} + Y^\top Q_{21}}{X^\top Q_{12} + Y^\top Q_{22}}{x_q^\top Q_{11}}{x_q^\top Q_{12}} \tbt{X}{x_q}{Y}{0}\right)   \\
    &= Z + \frac{1}{n} \tbt{P_{11} X + P_{21} Y}{P_{11} x_q}{P_{21} X + P_{22} Y}{P_{21} x_q} \tbt{*}{X^\top Q_{11} x_q + Y^\top Q_{21} x_q}{0}{0}.
\end{align*}
}
So, $\hat{y} = \mathsf{attn}(Z)[\bar{d}+1, n+1] = \frac{1}{n} (P_{21} X + P_{22} Y)(X^\top Q_{11} + Y^\top Q_{21})x_q$.
If we let
\[H = \frac{1}{n} \sum_{i=1}^n \tbo{\bar{x}_i}{y_i} [\bar{x}_i^\top \; y_i],\]
$b^\top$ be the last row of $\mathbf{P}$, and $A$ be the first $\bar{d}$ columns of $\mathbf{Q}$, then we have
\[\hat{y} = b^\top H A \bar{x}_{n+1}.\]
The ICL loss can be written as
    \[\mathcal{L}(\mathsf{BTFB}) =: L(b, A) = \EE_Z[(b^\top H A \bar{x}_{n+1} + \bar{w}^\top \bar{x}_{n+1})^2].\]
And we can set every entry of $\mathbf{P}, \mathbf{Q}$ outside of $b$ and $A$ as zero.

Define $\lambda_j = \EE[\bar{x}_{n+1}[j]^2]$.
And note that from Lemma~\ref{thm:quadratic-basis}, $\EE[\bar{x}_{n+1}[j]\bar{x}_{n+1}[j']] = 0$ for any $j \neq j'$.
Let $A_j$ be the $j$th column of $A$ and denote $G_j = b A_j^\top$, then we can write the loss function as:
\[ L(b, A) = \sum_{j=1}^{\bar{d}} \lambda_j \EE \left[(b^\top H A + w^\top)[j]^2\right] = \sum_{j=1}^{\bar{d}} \lambda_j \EE \left[(b^\top H A_j + \bar{w}_j)^2\right] = \sum_{j=1}^{\bar{d}} \lambda_j \EE \left[(\tr(H G_j^\top) + \bar{w}_j)^2\right].\]

Consider each of $j = 1, \dots, \bar{d}$.
Let $C^{(j)} := \argmin_C \EE \left[(\tr(H C^\top) + \bar{w}_j)^2\right].$
Then, we can write
\begin{align*}
    \tr(H {C^{(j)}}^\top)
    &= \sum_{k=1}^{\bar{d}}\sum_{\ell=1}^{\bar{d}} C^{(j)}_{k\ell} H_{k\ell} + \sum_{k=1}^{\bar{d}} C^{(j)}_{k, \bar{d}+1} H_{k, \bar{d}+1} + \sum_{k=1}^{\bar{d}} C^{(j)}_{\bar{d}+1, k} H_{\bar{d}+1, k} + C^{(j)}_{\bar{d}+1,\bar{d}+1} H_{\bar{d}+1,\bar{d}+1} \\
    &= \underbrace{\sum_{k=1}^{\bar{d}}\sum_{\ell=1}^{\bar{d}} \frac{C^{(j)}_{k\ell}}{n} \sum_i \bar{x}_i[k]\bar{x}_i[\ell]}_{\tau_1} + \underbrace{\sum_{k=1}^{\bar{d}} \frac{C^{(j)}_{\bar{d}+1, k} + C^{(j)}_{k, \bar{d}+1}}{n} \sum_i \bar{x}_i[k] y_i}_{\tau_2} + \underbrace{\frac{C^{(j)}_{\bar{d}+1,\bar{d}+1}}{n} \sum_i y_i^2}_{\tau_3}.
\end{align*}
It follows that
\begin{align*}
    \EE \left[(\tr(H {C^{(j)}}^\top) + \bar{w}_j)^2\right]
    &= \EE[(\tau_1 + \tau_3)^2] + \EE[(\tau_2 + \bar{w}_j)^2] + 2 \EE[(\tau_1 + \tau_3)(\tau_2 + \bar{w}_j)] \\
    &= \EE[(\tau_1 + \tau_3)^2] + \EE[(\tau_2 + \bar{w}_j)^2] \\
    &\ge \EE[(\tau_2 + \bar{w}_j)^2].
\end{align*}
Recall the expression for $y_i$'s in \eqref{equ:restruct-quadratic}, we know that $\tau_1, \tau_3$ are even in $w$ and $\tau_2, \bar{w}_j$ are odd in $w$.
Then, due to the symmetry of $w$, the cross term in the equation above is 0.

Therefore, we must have $\tau_1 = \tau_3 = 0$, which means that the entries of $C$ are zero except on the last row and column.
Then, we can achieve the minimum of $\EE \left[(\tr(H G_j^\top) + \bar{w}_j)^2\right]$ by setting
\[b = [0, \dots, 0, 1], A_j = \left[C^{(j)}_{\bar{d}+1, 1} + C^{(j)}_{1, \bar{d}+1}, \; \dots, \; C^{(j)}_{\bar{d}+1, \bar{d}} + C^{(j)}_{\bar{d},\bar{d}+1}, \; 0\right].\]

Repeating this analysis to all $j = 1, \dots, \bar{d}$ leads to the conclusion that the optimal attention weights can be written as
\[\mathbf{P} = \tbt{0_{d \times d}}{0}{0}{1}, \mathbf{Q} = \tbt{\Gamma}{0}{0}{0}.\]

\paragraph{Part 2:}
Now, given the form of the optimal attention weights $\mathbf{P}, \mathbf{Q}$, we optimize for $\Gamma$.
Let the minimizer be $\Gamma^\star$ and we write $\Gamma^\star = - \EE[\bar{x}\bar{x}^\top]^{-1} + \Delta$.

Recall that from Lemma~\ref{thm:attn-loss}, for any fixed $\bar{w}$, the attention layer's prediction error can be written as
\[\EE_x\left[(\hat{y} + y_{n+1})^2\right] = \min_\beta \EE\left(\inner{\beta}{\bar{x}_{n+1}} + y_{n+1}\right)^2 + \frac{1}{n} \tr\left(\Gamma \Lambda \Gamma^\top \Phi \right) + \xi^\top (\Gamma^\top + {\Lambda^{-1}})^\top \Lambda (\Gamma^\top + \Lambda^{-1}) \xi,\]
where 
\[ y_i = \inner{\bar{x}_i}{\bar{w}}, \Lambda = \EE_x[\bar{x} \bar{x}^\top], \Phi = \mathrm{Cov}_x[\bar{x}y], \xi = \EE_x[\bar{x}y]. \]
Note that because $\Phi$ and $\Lambda$ are positive semidefinite,
\[\tr\left(\Gamma \Lambda \Gamma^\top \Phi \right) = \tr\left(\Phi^{1/2} \Gamma \Lambda \Gamma^\top \Phi^{1/2} \right) = \tr\left((\Phi^{1/2} \Gamma \Lambda^{1/2}) (\Phi^{1/2} \Gamma \Lambda^{1/2})^\top\right) \ge 0.\]
Therefore,
\[\EE_x\left[(\hat{y} + y_{n+1})^2\right] \ge \xi^\top (\Gamma^\top + {\Lambda^{-1}})^\top \Lambda (\Gamma^\top + \Lambda^{-1})\xi.\]
After taking the expectation over $w$, we find that the minimum ICL loss is:
\begin{flalign*}
    && \min_\Gamma \mathcal{L}
    &\ge \EE_w\left[\xi^\top ({\Gamma^\star}^\top + {\Lambda^{-1}})^\top \Lambda ({\Gamma^\star}^\top + \Lambda^{-1}) \xi\right] \\
    && &= \EE_w\left[\tr\left(\Delta \Lambda \Delta^\top \xi\xi^\top\right)\right] \\
    && &= \tr\left(\Delta \Lambda \Delta^\top \EE[\xi\xi^\top]\right) && \vartriangleright \text{ Define } \Sigma := \EE[\xi\xi^\top]\\
    && &= \tr\left((\Lambda^{1/2} \Delta^\top \Sigma^{1/2})^\top(\Lambda^{1/2} \Delta^\top \Sigma^{1/2})\right) \\
    && &= \norm{\Lambda^{1/2} \Delta^\top \Sigma^{1/2}}_F^2.
\end{flalign*}

And by Theorem~\ref{thm:bilinear-construction}, we must have
\[\norm{\Lambda^{1/2} \Delta^\top \Sigma^{1/2}}_F^2 \le \frac{(d+2)(d+1) + d}{2n}.\]
Because of the matrix norm identity $\norm{AB}_F \le \norm{A}_F \norm{B}_{\mathrm{op}}$, we have
\[\norm{\Delta}_F = \norm{\Delta^\top}_F \le \norm{\Lambda^{1/2} \Delta^\top \Sigma^{1/2}}_F \norm{\Lambda^{-1/2}}_{\mathrm{op}} \norm{\Sigma^{-1/2}}_{\mathrm{op}}.\]
Now, it suffices to find the smallest eigenvalue of $\Sigma$ and $\Lambda$.

First, due to Lemma~\ref{thm:quadratic-basis}, $\Lambda = \EE[\bar{x}\bar{x}^\top]$ is diagonal with nonzero entries equal to either 1 or 2.
Thus, $\norm{\Lambda^{-1/2}}_{\mathrm{op}} = \max(1/\sqrt{2}, 1) = 1$.
Next, we directly compute $\xi$, where we apply the expression for $y$ as written in \eqref{equ:restruct-quadratic},
\begin{align*}
    \xi = \EE_x[\bar{x}y] 
    = \EE_x[\bar{x}\inner{\bar{x}}{\bar{w}}] 
    = \left[\EE_x[\bar{x}[k]\inner{\bar{x}}{\bar{w}}] \vphantom{x^2}\right]_{k=1}^{\bar{d}}
    = \bigg[\EE_x[\bar{x}[k]^2]\bar{w}[k]\bigg]_{k=1}^{\bar{d}}
    = \Lambda \bar{w}.
\end{align*}
Taking the expectation over $w$ yields that
\[\Sigma = \Lambda \EE_w[\bar{w}\bar{w}^\top] \Lambda.\]
And since $w \sim \mathcal{N}(0, \mathbf{I}_{\bar{d}})$, the covariance matrix $\EE_w[\bar{w}\bar{w}^\top]$ after some permutation of the coordinates can be written as:
\[\begin{bmatrix}
    d+1 & \mathbf{1}_{1 \times d} & \mathbf{0} \\
    \mathbf{1}_{d \times 1} & \mathbf{I}_{d} & \mathbf{0} \\
    \mathbf{0} & \mathbf{0} & \mathbf{I}_{\bar{d}-d-1}
\end{bmatrix},\]
where the top-left $(d+1) \times (d+1)$ submatrix correspond to the coefficients $\bar{w}_{00}, \bar{w}_{11}, \dots, \bar{w}_{dd}$.
The eigenvalues of this top-left submatrix are $\frac{1}{2}(d+2 \pm \sqrt{d^2+4d}), \, 1, \, \dots, \, 1$ with corresponding eigenvectors
\[[\frac{1}{2}(d \pm \sqrt{d^2+4d}), 1, \dots 1], [0, -1, 1, 0, \dots], [0, -1, 0, 1, 0, \dots], \dots, [0, -1, 0, \dots, 1].\]
Therefore, we have
\[\norm{\Sigma^{-1/2}}_{\mathrm{op}} \le \norm{\Lambda^{-1}}_{\mathrm{op}} \sqrt{\frac{2}{d+2-\sqrt{d^2+4d}}} = \sqrt{\frac{2}{d+2-\sqrt{d^2+4d}}}.\]
We can approximate the RHS by a Taylor expansion:
\begin{align*}
    \sqrt{\frac{2}{d+2-\sqrt{d^2+4d}}}
    &= \sqrt{\frac{2/(d+2)}{1-\sqrt{1-4/(d+2)^2}}} \\
    &= \sqrt{\frac{2/(d+2)}{2/(d+2)^2 - O(d^{-4})}} \\
    &= \sqrt{(d+2)\frac{1}{1-O(d^{-2})}} \\
    &= \sqrt{d+2 + O(d^{-2})}.
\end{align*}

In conclusion, we have
\begin{align*}
    \norm{\Delta}_F 
    &\le \norm{\Lambda^{1/2} \Delta^\top \Sigma^{1/2}}_F \norm{\Lambda^{-1/2}}_{\mathrm{op}} \norm{\Sigma^{-1/2}}_{\mathrm{op}} \\
    &= \sqrt{\frac{(d+2)(d+1) + d}{2n}} \cdot \sqrt{\frac{2}{d+2-\sqrt{d^2+4d}}} \\
    &= n^{-1/2} \sqrt{\frac{(d+2)(d+1) + d}{d+2-\sqrt{d^2+4d}}} \\
    &\in O\left(\frac{d^{3/2}}{\sqrt{n}}\right),
\end{align*}
as desired.
 
\section{Proof of Theorem~\ref{thm:block-construction}}
\label{sec:app-block-construction}
\begingroup
\renewcommand*{\arraystretch}{1.25}

We shall state a more complete version of Theorem~\ref{thm:block-construction}, with the construction explicitly specified.
Note that we use \% to denote the remainder operator, with $a \% b$ having the range $\{1, \dots, b\}$. 

\begin{theorem}
    For the quadratic function in-context learning problem as described in Section~\ref{sec:in-context}, we consider a $\mathsf{TF_{bilin}}$ with $2L$ layers and $\bar{d} = 2d+1$.
    For $\ell = 1, \dots, L$, we choose the $\ell$th bilinear feed-forward layer weights so that
    \[\mathbf{W}_0^{(\ell)} = 
    \begin{bmatrix}
        \mathbf{0}_{(d+1) \times 1} &\mathbf{0}_{(d+1)\times d} & \mathbf{0}_{(d+1)\times d} \\
        \mathbf{0}_{d \times 1} &\widetilde{\mathbf{W}}_0^{(\ell)} & \mathbf{0}_{d \times d}
    \end{bmatrix}, 
    \mathbf{W}_1^{(\ell)} = 
    \begin{bmatrix}
        \mathbf{0}_{(d+1) \times 1} &\mathbf{0}_{(d+1)\times d} & \mathbf{0}_{(d+1) \times d} \\
        \mathbf{0}_{d \times 1} &\widetilde{\mathbf{W}}_1^{(\ell)} & \mathbf{0}_{d \times d}
    \end{bmatrix},
    \]
    with
    \[\widetilde{\mathbf{W}}_0^{(\ell)} = I_d, \widetilde{\mathbf{W}}_1^{(\ell)} = 
    \begin{cases}
        E_1 & \text{if } \ell = 1, \\
        -E_{(\ell-1) \% d} + E_{\ell \% d} & \text{otherwise},\\
    \end{cases}\]
    where $E_j$ denotes the matrix where the $j$th column is equal to all 1's and 0 everywhere else.
    And we choose the $\ell$th self-attention layer weights as
    \[\mathbf{P}^{(2\ell)} = \tbt{0_{d \times d}}{0}{0}{1}, \mathbf{Q}^{(2\ell)} = \tbt{\Gamma^{(\ell)}}{0}{0}{0}.\]

    Then, for $\ell = 1, \dots, L$, the forward pass of $\mathsf{TF_{bilin}}$ follows the form:
    
    \begin{align*}
        Z^{(0)} &= Z =: \begin{bmatrix}
        1 & 1 & \cdots & 1 \\
        x_1 & x_2 & \cdots & x_{n+1} \\
        0 & 0 & \cdots & 0 \\
        y^{(0)}_1 & y^{(0)}_2 & \cdots & y^{(0)}_{n+1}
        \end{bmatrix}, \\
        Z^{(2\ell-1)} &= \begin{bmatrix}
        1 & 1 & \cdots & 1 \\
        x_1 & x_2 & \cdots & x_{n+1} \\
        x_1[\ell \% d] \cdot x_1 & x_2[\ell \% d] \cdot x_2 & \cdots & x_{n+1}[\ell \% d] \cdot x_{n+1} \\
        y^{(\ell-1)}_1 & y^{(\ell-1)}_2 & \cdots & y^{(\ell-1)}_{n+1}
        \end{bmatrix}, \\ \\
        Z^{(2\ell)} &= \begin{bmatrix}
        1 & 1 & \cdots & 1 \\
        x_1 & x_2 & \cdots & x_{n+1} \\
        x_1[\ell \% d] \cdot x_1 & x_2[\ell \% d] \cdot x_2 & \cdots & x_{n+1}[\ell \% d] \cdot x_{n+1} \\
        y^{(\ell)}_1 & y^{(\ell)}_2 & \cdots & y^{(\ell)}_{n+1}
        \end{bmatrix}.
    \end{align*}
    
    Furthermore, $\mathsf{TF_{bilin}}$ implements block-coordinate descent over a quadratic regression in the sense that
    \[\hat{y}_i^{(\ell)} = y_i + \obt{1}{x_{\mathsf{query}}} W^{(\ell)} \tbo{1}{x_{\mathsf{query}}}, \]
    where $W^{(\ell)}$ are iterates of the block-coordinate descent update \eqref{equ:block-update} with initialization $W^{(0)} = 0$, step sizes $\eta^{(\ell)} = -{\Gamma^{(\ell)}}^\top$ and the sequence of feature subsets can be written as
    \[ (1, x_i[1], \dots, x_i[d], x_i[1] x_i[\ell], \dots, x_i[d] x_i[\ell]). \]
\end{theorem}

\begin{proof}
    We shall prove the claim through induction.
    First, we consider the bilinear feed-forward layers.
    For the $\ell$-th bilinear layer, its input is $Z^{(2\ell-2)}$, and with our chosen weights, we have
    {\small
    \begin{align*}
    Z^{(1)} 
    &= Z^{(0)} + \left(\tbt{\mathbf{W_1}^{(1))}}{0}{0}{0} Z^{(0)}\right) \odot \left(\tbt{\mathbf{W_2}^{(1)}}{0}{0}{0} Z^{(0)}\right) \\
    &= 
    \begin{bmatrix}
        1 & 1 & \cdots & 1 \\
        x_1 & x_2 & \cdots & x_{n+1} \\
        0 & 0 & \cdots & 0 \\
        y^{(0)}_1 & y^{(0)}_2 & \cdots & y^{(0)}_{n+1}
    \end{bmatrix}
    + \begin{bmatrix}
        1 & 1 & \cdots & 1 \\
        x_1 & x_2 & \cdots & x_{n+1} \\
        x_1[1] \cdot x_1 & x_2[1] x_2 & \cdots & x_{n+1}[1]\cdot x_{n+1} \\
        y^{(0)}_1 & y^{(0)}_2 & \cdots & y^{(0)}_{n+1}
    \end{bmatrix}\\
    &= \begin{bmatrix}
        1 & 1 & \cdots & 1 \\
        x_1 & x_2 & \cdots & x_{n+1} \\
        x_1[1] \cdot x_1 & x_2[1] x_2 & \cdots & x_{n+1}[1]\cdot x_{n+1} \\
        y^{(0)}_1 & y^{(0)}_2 & \cdots & y^{(0)}_{n+1}
    \end{bmatrix}.
    \end{align*}
    }
    and for $\ell > 1$,
    {\footnotesize
    \begin{align*}
    Z^{(2\ell-1)} 
    &= Z^{(2\ell-2)} + \left(\tbt{\mathbf{W_1}^{(\ell)}}{0}{0}{0} Z^{(2\ell-2)}\right) \odot \left(\tbt{\mathbf{W_2}^{(\ell)}}{0}{0}{0} Z^{(2\ell-2)}\right) \\
    &= 
    \begin{bmatrix}
        1 & 1 & \cdots & 1 \\
        x_1 & x_2 & \cdots & x_{n+1} \\
        x_1[(\ell-1) \% d] \cdot x_1 & x_2[(\ell-1) \% d] \cdot x_2 & \cdots & x_{n+1}[(\ell-1) \% d] \cdot x_{n+1} \\
        y^{(\ell-1)}_1 & y^{(\ell-1)}_2 & \cdots & y^{(\ell-1)}_{n+1}
    \end{bmatrix} \\
    &\quad +
    \begin{bmatrix}
        1 & 1 & \cdots & 1 \\
        x_1 & x_2 & \cdots & x_{n+1} \\
        (x_1[\ell \% d] - x_1[(\ell-1) \% d]) \cdot x_1 & (x_2[\ell \% d] - x_2[(\ell-1) \% d]) \cdot x_2 & \cdots & (x_{n+1}[\ell \% d] - x_{n+1}[(\ell-1) \% d])\cdot x_{n+1} \\
        y^{(\ell-1)}_1 & y^{(\ell-1)}_2 & \cdots & y^{(\ell-1)}_{n+1}
    \end{bmatrix}\\
    &= \begin{bmatrix}
        1 & 1 & \cdots & 1 \\
        x_1 & x_2 & \cdots & x_{n+1} \\
        x_1[\ell \% d] \cdot x_1 & x_2[\ell \% d] \cdot x_2 & \cdots & x_{n+1}[\ell \% d] \cdot x_{n+1} \\
        y^{(\ell-1)}_1 & y^{(\ell-1)}_2 & \cdots & y^{(\ell-1)}_{n+1}
    \end{bmatrix}.
    \end{align*}
    }

    Next, we turn our focus to the self-attention layers.
    For the $\ell$-th self-attention layer, its input is $Z^{(2\ell-1)}$.
    Denote
    \[Z^{(2\ell-1)} =: \tbt{X}{x_q}{Y}{y_q},
    \]
    where the $X \in \RR^{\bar{d} \times n}, Y \in \RR^{1 \times n}, x_q \in \RR^{\bar{d}}$ and $y_q \in \RR$.
    {\small
    \begin{align*}
        Z^{(2\ell)}
        &= Z^{(2\ell-1)} + \frac{1}{n} \left(\mathbf{P}^{(\ell)} Z^{(2\ell-1)}\right) \tbt{I_{n}}{0}{0}{0} \left({Z^{(2\ell-1)}}^\top \mathbf{Q}^{(\ell)} Z^{(2\ell-1)}\right)   \\
        &= Z^{(2\ell-1)} + \frac{1}{n} \tbt{0_{d\times d}}{0}{0}{1} \tbt{X}{x_q}{Y}{y_q} \tbt{I_{n}}{0}{0}{0} \left(\tbt{X^\top}{Y^\top}{x_q^\top}{y_q} \tbt{\Gamma^{(\ell)}}{0}{0}{0} \tbt{X}{x_q}{Y}{y_q}\right) \\
        &= Z^{(2\ell-1)} + \frac{1}{n} \tbt{0}{0}{Y}{y_q} \tbt{I_n}{0}{0}{0} \tbt{X^\top}{V^\top}{x_q^\top}{y_q} \tbt{\Gamma^{(\ell)} X}{\Gamma^{(\ell)} x_q}{0}{0} \\
        &= Z^{(2\ell-1)} + \frac{1}{n} \tbt{0}{0}{Y}{y_q} \tbt{X^\top \Gamma X}{X^\top \Gamma^{(\ell)} x_q}{0}{0} \\
        &= \tbt{X}{x_q}{Y + \frac{1}{n} Y X^\top \Gamma^{(\ell)} X}{y_q + \frac{1}{n} Y X^\top \Gamma^{(\ell)} x_q} \\
        &= \begin{bmatrix}
        1 & 1 & \cdots & 1 \\
        x_1 & x_2 & \cdots & x_{n+1} \\
        x_1[\ell \% d] \cdot x_1 & x_2[\ell \% d] \cdot x_2 & \cdots & x_{n+1}[\ell \% d] \cdot x_{n+1} \\
        y^{(\ell)}_1 & y^{(\ell)}_2 & \cdots & y^{(\ell)}_{n+1}
        \end{bmatrix},
    \end{align*}
    }
    where 
    \begin{align*}
        y_{i}^{(\ell)} &= y_{i}^{(\ell-1)} + \inner{\bar{x}_i^{(\ell)}}{\frac{1}{n}\sum_{j=1}^n{\Gamma^{(\ell)}}^\top \bar{x}_j^{(\ell)} y_{j}^{(\ell-1)}}, \\
        \bar{x}_i^{(\ell)} &= (1, x_i, x_i[\ell \% d] \cdot x_i).
    \end{align*}

    Now we compare $y^{(\ell)}$ against the iterates of block-coordinate descent through an induction argument.
    Let
    \[f(x; W) = \obt{1}{x} W \tbo{1}{x}.\]
    Note that for any $j \in b_{\ell}$:
    \begin{align*}
        \partial w_j \left(\frac{1}{2n} \sum_{i=1}^n \left(f(x_i, W^{(\ell-1)}) + y_i\right)^2\right)
        &= \frac{1}{n} \sum_{i=1}^n \left(f(x_i, W^{(\ell-1)}) + y_i\right) \widetilde{x}_i[j] \\
        &= \frac{1}{n} \sum_{i=1}^n y^{(\ell-1)}_i \widetilde{x}_i[j] \\
        &= \frac{1}{n} \sum_{i=1}^n y^{(\ell-1)}_i \bar{x}^{(\ell)}_i[j],
    \end{align*}
    where we use $\widetilde{x}_i$ to denote the collection all quadratic monomials on $x_i$ and note that $\bar{x}^{(\ell)}_i = \widetilde{x}_i[b_\ell]$.
    Therefore,
    \begin{align*}
        y_i + f(x_i, W^{(\ell)})
        &= y_i + f(x_i, W^{(\ell-1)}) - \inner{\bar{x}_i^{(\ell)}}{\frac{1}{n} \sum_{i=1}^n \eta^{(\ell)}\bar{x}_i y^{(\ell-1)}_i},
    \end{align*}
    which by induction is equal to $y_i^{(\ell)}$ if we set $\eta^{(\ell)} = -{\Gamma^{(\ell)}}^\top$. 
    So we are done.
\end{proof}

\endgroup

\section{Proof of Corollary~\ref{thm:poly-block-construction}}
\label{sec:proof-poly-block-construction}
Because the attention layers constructed in Theorem~\ref{thm:block-construction} only modify the last row, the only change from the proof of Theorem~\ref{thm:block-construction} is to design a sequence of feature subsets that cycle through all degree-$p$ polynomial features.

Recall that, for matrix parameters $\mathbf{W}_0, \mathbf{W}_1 \in \RR^{\bar{d} \times \bar{d}}$, we define the bilinear layer as
\begin{equation*}
    \mathsf{bilin}(Z) = Z + \left(\tbt{\mathbf{W}_0}{\mathbf{0}}{\mathbf{0}}{0} Z\right) \odot \left(\tbt{\mathbf{W}_1}{\mathbf{0}}{\mathbf{0}}{0} Z\right),
\end{equation*}
where $\odot$ denotes the Hadamard product.
Then, we partition the weights into submatricies of the following sizes:
\[
    \begin{bmatrix}
        (d+1) \times 1 & (d+1)\times d & (d+1)\times d \\
        d \times 1 & d \times d & d \times d
    \end{bmatrix}.
\]
Under this partition, we reserve the first $d+1$ dimensions of the feature space to the original variables $(1, x[1], \dots, x[d])$ and use the remaining $\bar{d}-d-1 = d$ dimensions as the scratchpad for computing nonlinear features.
We can define the feed-forward layers to perform the following useful operations of the features:
\begin{list}{{\tiny $\blacksquare$}}{\leftmargin=1em}
\setlength{\topsep}{-1pt}
\setlength{\itemsep}{-3pt}
    \item We can erase all of the currently computed nonlinear features:
    \[\mathbf{W}_0 = 
    \begin{bmatrix}
        \mathbf{0}_{(d+1) \times 1} &\mathbf{0}_{(d+1)\times d} & \mathbf{0}_{(d+1)\times d} \\
        \mathbf{0}_{d \times 1} & \mathbf{0}_{d \times d} & \mathbf{I}_d
    \end{bmatrix},
    \mathbf{W}_1 = 
    \begin{bmatrix}
        \mathbf{0}_{(d+1) \times 1} &\mathbf{0}_{(d+1)\times d} & \mathbf{0}_{(d+1) \times d} \\
        -\mathbf{1}_{d \times 1} & \mathbf{0}_{d \times d} & \mathbf{0}_{d \times d}
    \end{bmatrix}.
    \]
    \item If the scratchpad space is all 0's, we can generate quadratic features of the form $x[j]x[1], \dots, x[j]x[d]$:
    \[\mathbf{W}_0 = 
    \begin{bmatrix}
        \mathbf{0}_{(d+1) \times 1} &\mathbf{0}_{(d+1)\times d} & \mathbf{0}_{(d+1)\times d} \\
        \mathbf{0}_{d \times 1} & \mathbf{I}_0 & \mathbf{0}_{d \times d}
    \end{bmatrix}, 
    \mathbf{W}_1 = 
    \begin{bmatrix}
        \mathbf{0}_{(d+1) \times 1} &\mathbf{0}_{(d+1)\times d} & \mathbf{0}_{(d+1) \times d} \\
        \mathbf{0}_{d \times 1} &  \mathbf{E}_j & \mathbf{0}_{d \times d}
    \end{bmatrix},
    \]
    where $\mathbf{E}_j$ denotes the matrix where the $j$th column is all 1's and 0 everywhere else.
    \item Lastly, we can multiply every existing nonlinear feature by the variable $x[j]$:
    \[
    \mathbf{W}_0 =
        \begin{bmatrix}
        \mathbf{0}_{(d+1) \times 1} &\mathbf{0}_{(d+1)\times d} & \mathbf{0}_{(d+1)\times d} \\
        \mathbf{0}_{d \times 1} & \mathbf{0}_{d \times d} & \mathbf{I}_d
    \end{bmatrix},
    \mathbf{W}_1 = 
    \begin{bmatrix}
        \mathbf{0}_{(d+1) \times 1} &\mathbf{0}_{(d+1)\times d} & \mathbf{0}_{(d+1) \times d} \\
        -\mathbf{1}_{d \times 1} &  \mathbf{E}_j & \mathbf{0}_{d \times d}
    \end{bmatrix}.
    \]
\end{list}
We note that the original variables $(1, x[1], \dots, x[d])$ are always carried in the forward pass by the residual connection.

Now, let $g(x) = x[j_1]x[j_2]\cdots x[j_{p-1}]$ be a monomial with degree equal to $p-1$.
Then we can generate the features $x[1]g(x), \dots, x[d]g(x)$ over $p$ feed-forward layers.

\begin{list}{{\tiny $\blacksquare$}}{\leftmargin=1em}
\setlength{\topsep}{-1pt}
\setlength{\itemsep}{-3pt}
    \item We first clear the nonlinear features in the last $d$ dimensions.
    \item Then, in the next layer, we generate $x[1]x[j_1], \dots, x[d]x[j_1]$.
    \item For the next $p-2$ layers, we multiply the nonlinear features by $x[j_2], \dots, x[j_{p-1}]$, respectively.
\end{list}

We can repeat this procedure for all $\binom{d+p-2}{p-1}$ monomials with degree $p-1$.
Note that all monomials of lower degrees would be generated as part of the intermediate layers in the procedure above.
Therefore, we can cycle through all possible polynomial features with degree up to $p$ in $\binom{d+p-2}{p-1} p \le \max(d, p)^{\min(d, p)}$ bilinear Transformer blocks whose embedding dimensions are $\bar{d} = 2d+1$.

\section{Additional Experimental Results}
\label{sec:additional-experiment}

In this section, we present several additional plots from our experiments from our experiments on quadratic and cubic ICL tasks described in Section~\ref{sec:efficiency}.

First, we present a new comparison examining how the prompt length during training may affect the performance of the learned Transformers.
e consider a quadratic ICL setup similar to that in Figure\ref{fig:block-quadratic}, where we have $d=3$ and $\bar{d}=4$.
Using a bilinear Transformer with $8$ layers (which implements 4 steps of block coordinate descent), we train models using training prompts of length $n_{\rm train} \in \{100, 150, 200, 250\}$.
This yields one trained Transformer for each choice of $n_{\rm train}$.
We then compare their ICL test loss across test prompts of varying lengths
We find that larger values of $n_{\rm train}$ leads to lower ICL loss when the test prompt is long but performance degrades quickly when the test prompt length is shorter than $n_{\rm train}$.
Because the variance becomes quite large when the length of training prompts is much lower than $n_{\rm train}$, we omit the error bar for these parts and use dashed line to plot the average.
Furthermore, unlike other figures, the error bar in this figure represent half the standard deviation because the losses are very small --- even minor variations would produce excessively large error bars on a log scale.

\begin{figure}[!h]
    \centering
    \includegraphics[width=0.6\linewidth]{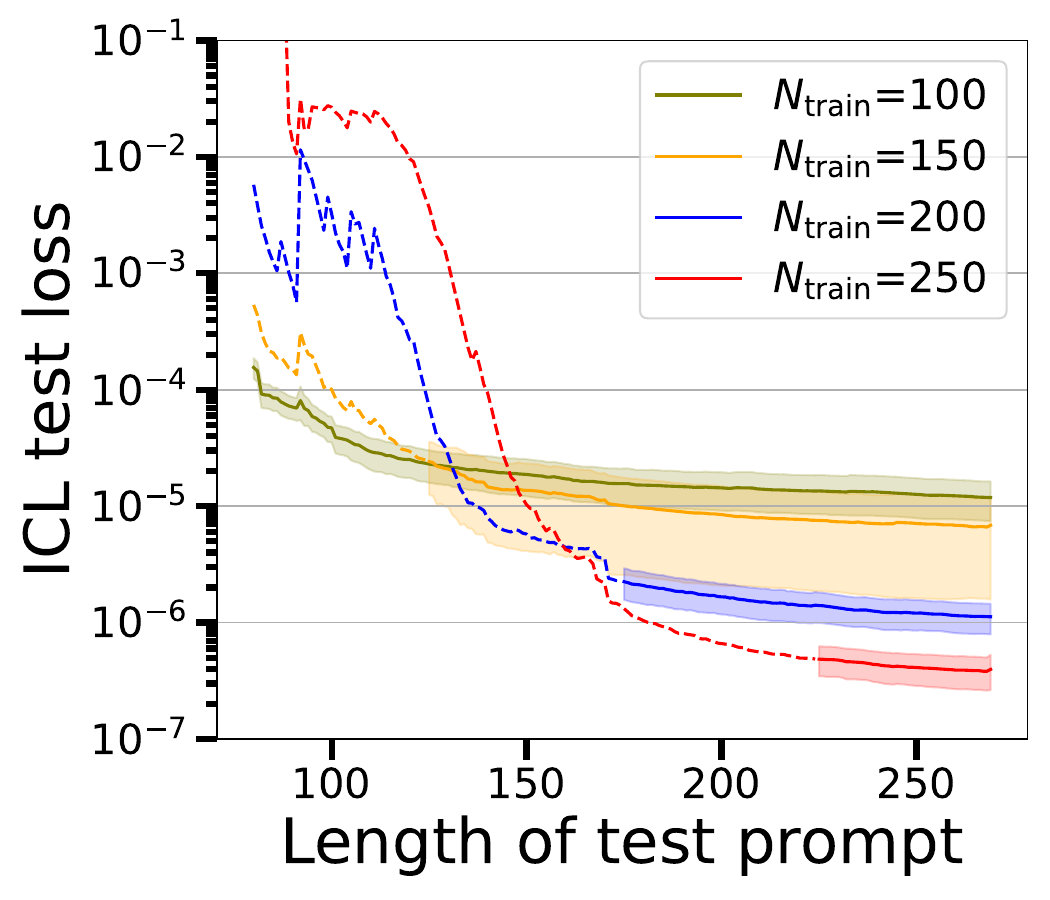}
    \caption{Quadratic ICL test loss for 8-layer bilinear Transformer models trained on training prompts of length $n_{\rm train} \in \{100, 150, 200, 250\}$, respectively.}
    \label{fig:gen-plot}
\end{figure}

\clearpage

For the various bilinear Transformer models in our experiments, we plot their ICL losses over 20000 training iterations.
At each iteration, the test loss is computed from a test prompt independently sampled from the same distribution as the training prompts.
We also inspect the iterates of the block-coordinate descent update implemented by the Transformers.
We compute predictions corresponding to different numbers of block-coordinate update steps by using the sub-network where the first $2\ell$ layers correspond to $\ell$ steps.

The next three figures pertain to the experiment on quadratic target functions with 4 variables ($d=4$).
For a Transformer with $2L$ layers, Theorem~\ref{thm:block-construction} indicates that it implements $L$ steps of block-coordinate descent.
We note the significant improvement in the ICL loss as we increase the depth of the Transformer.
We also observe that training shallower Transformer models exhibits less variation.

Due to the large degrees of freedom in the bilinear Transformer, the initial iterations may exhibit very large run-to-run variations.
Therefore, we omit the error bars for these early iterations and plot the average using a dashed line.
For the experiments on quadratic ICL, we omit the error bars for the first 3,000 iterations.

\begin{figure}[H]
    \centering
    \includegraphics[width=0.6\linewidth]{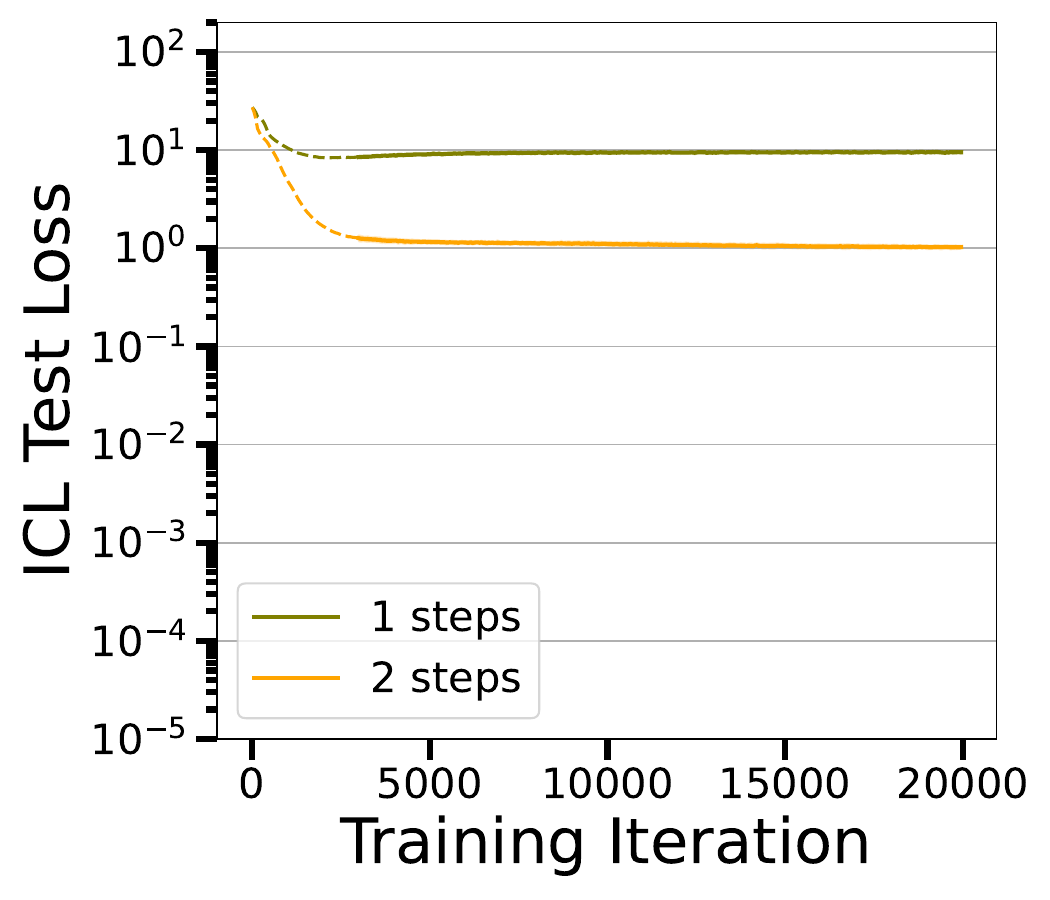}
    \caption{Loss curve (over 5 trials) for the quadratic ICL task, $d = 4, L = 2$.}
\end{figure}

\begin{figure}[H]
    \centering
    \includegraphics[width=0.6\linewidth]{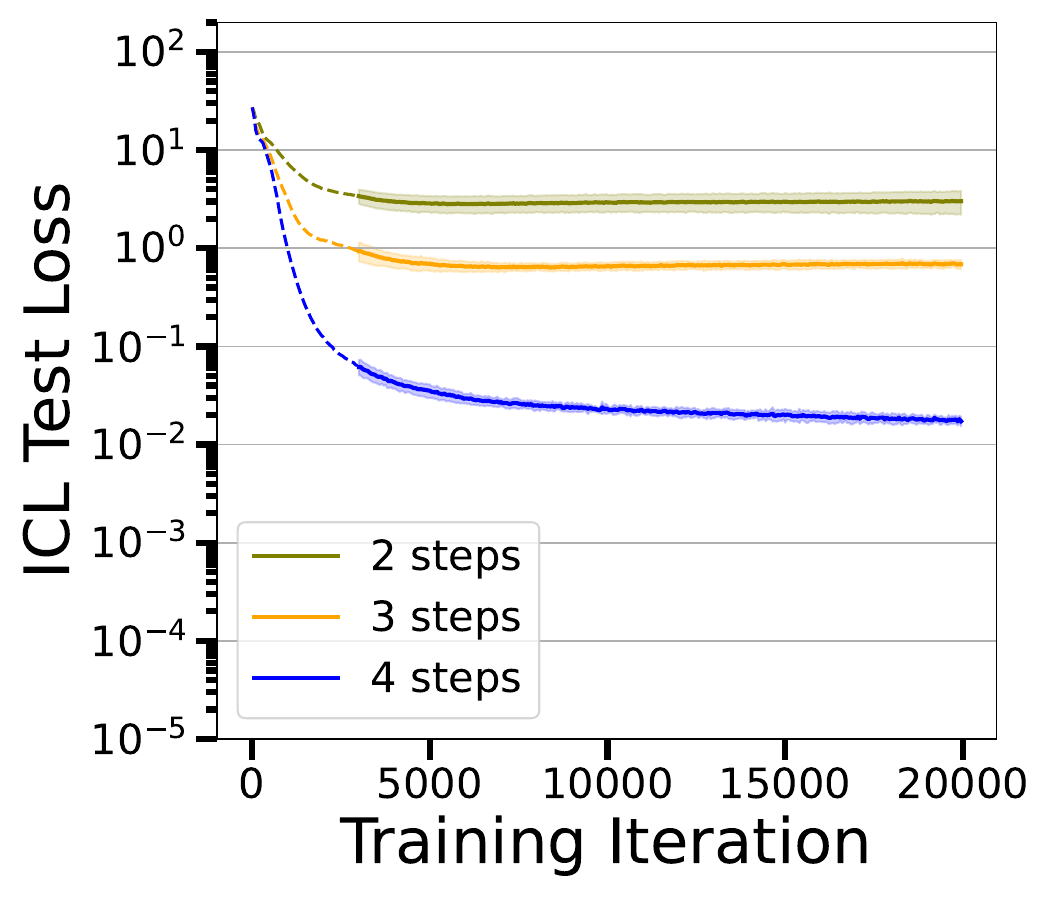}
    \caption{Loss curve (over 5 trials) for the quadratic ICL task, $d = 4, L = 4$.}
\end{figure}

\begin{figure}[H]
    \centering
    \includegraphics[width=0.6\linewidth]{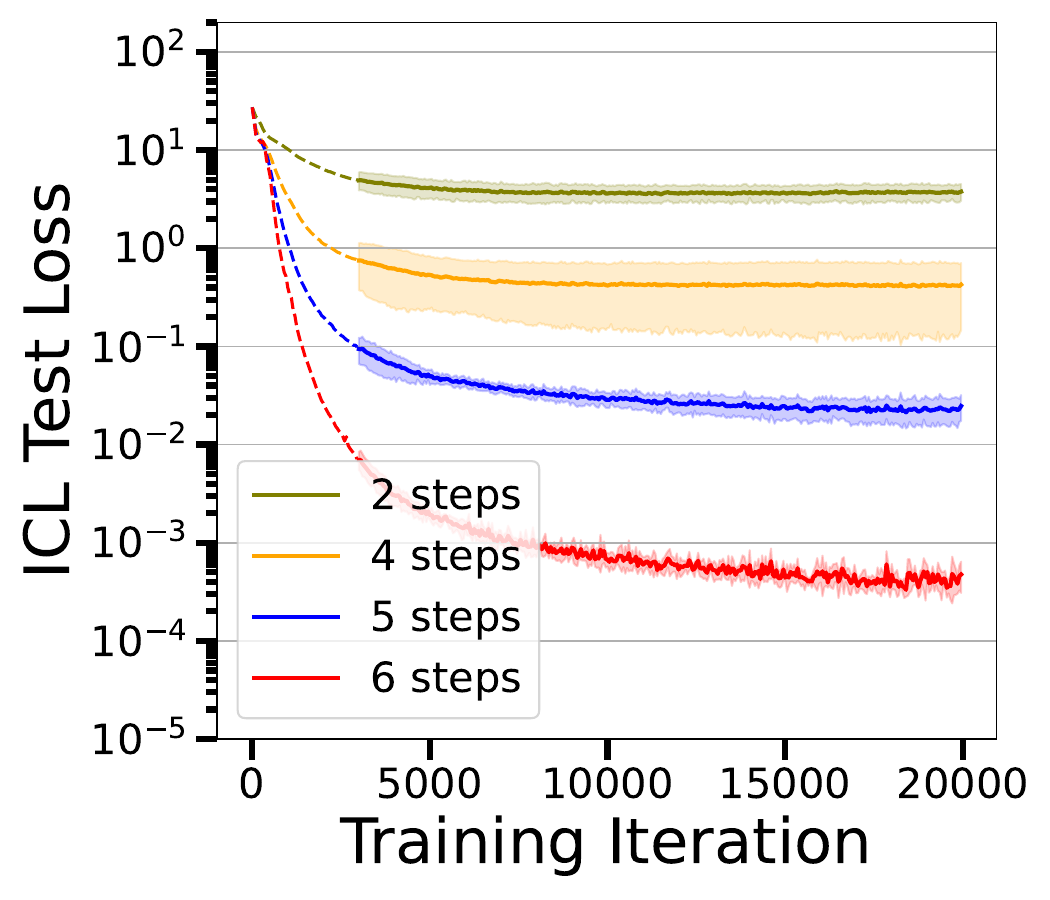}
    \caption{Loss curve (over 5 trials) for the quadratic ICL task, $d = 4, L = 6$.}
\end{figure}

The next two figures pertain to the experiments on quadratic target functions with 3 variables ($d=3$).
Similar observations to those in the $d=4$ case carry over here.
We note that the error bar in the plot for $L=4$ may appear large due to the log scale on the $y$-axis --- the actual standard deviation between trials is only about $10^{-5}$.

\begin{figure}[H]
    \centering
    \includegraphics[width=0.6\linewidth]{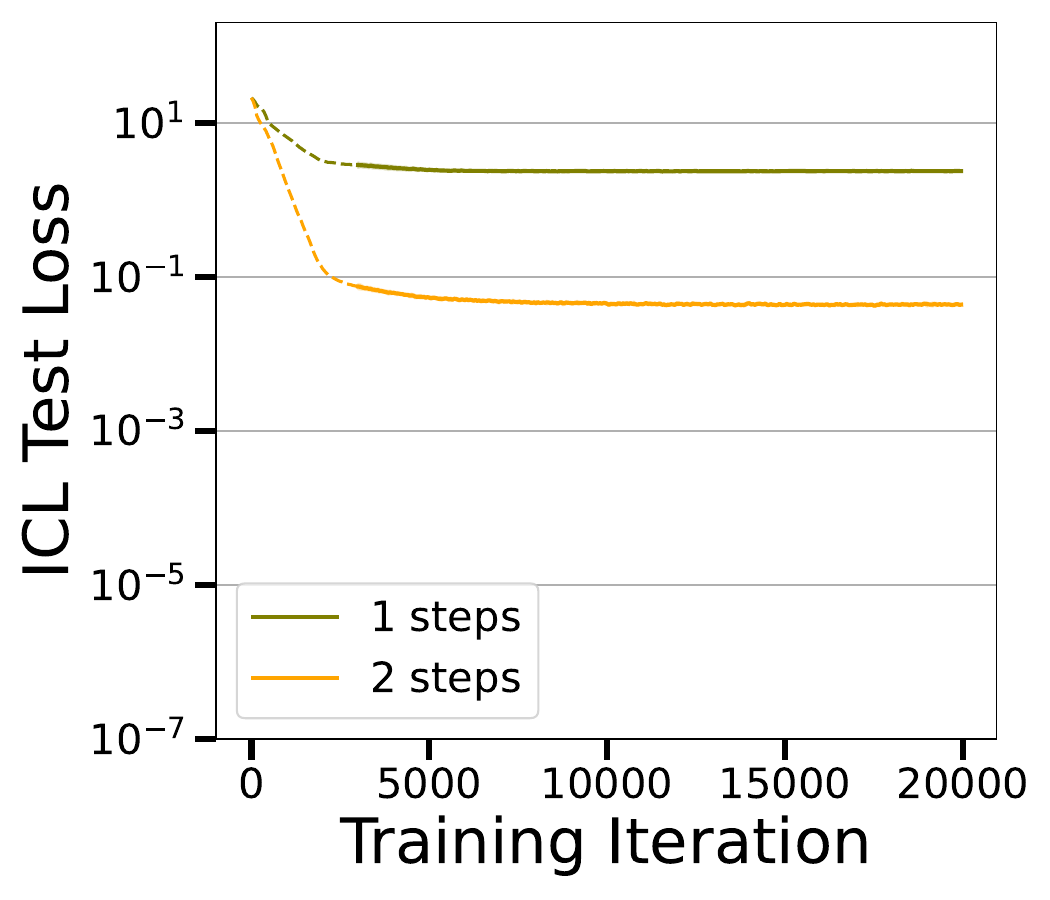}
    \caption{Loss curve (over 5 trials) for the quadratic ICL task, $d = 3, L = 2$.}
\end{figure}

\begin{figure}[H]
    \centering
    \includegraphics[width=0.6\linewidth]{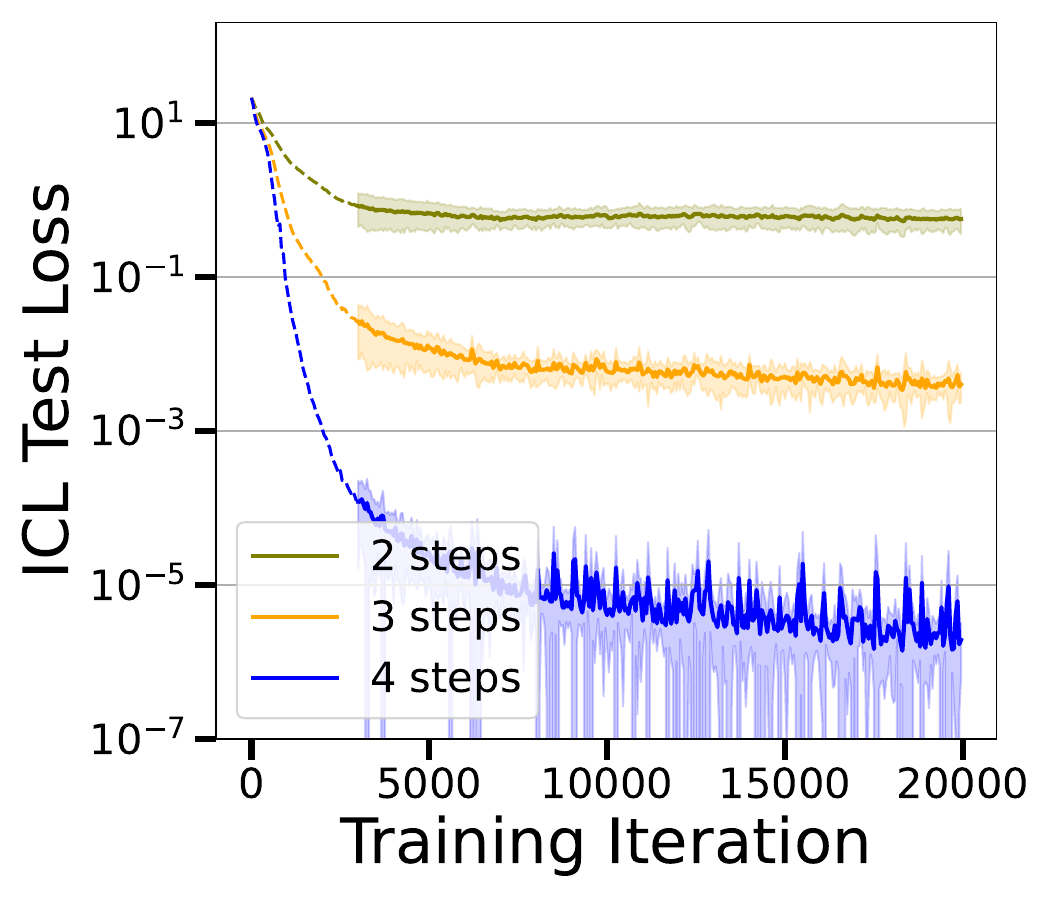}
    \caption{Loss curve (over 5 trials) for the quadratic ICL task, $d = 3, L = 4$.}
\end{figure}

The last figure pertains to the experiment on cubic target functions, where we compare the performance of bilinear Transformers with dense \eqref{equ:bilinear} and sparse~\ref{equ:sparse-bilinear} bilinear layers.
We find that the sparse version performs much worse, indicating that the sparse bilinear layers inhibit the learning of higher-order features in context.
Because there are even greater degrees of freedom for a cubic ICL task, we omit the error bar for the first 9000 iterations.

\begin{figure}[!h]
    \centering
    \includegraphics[width=0.6\linewidth]{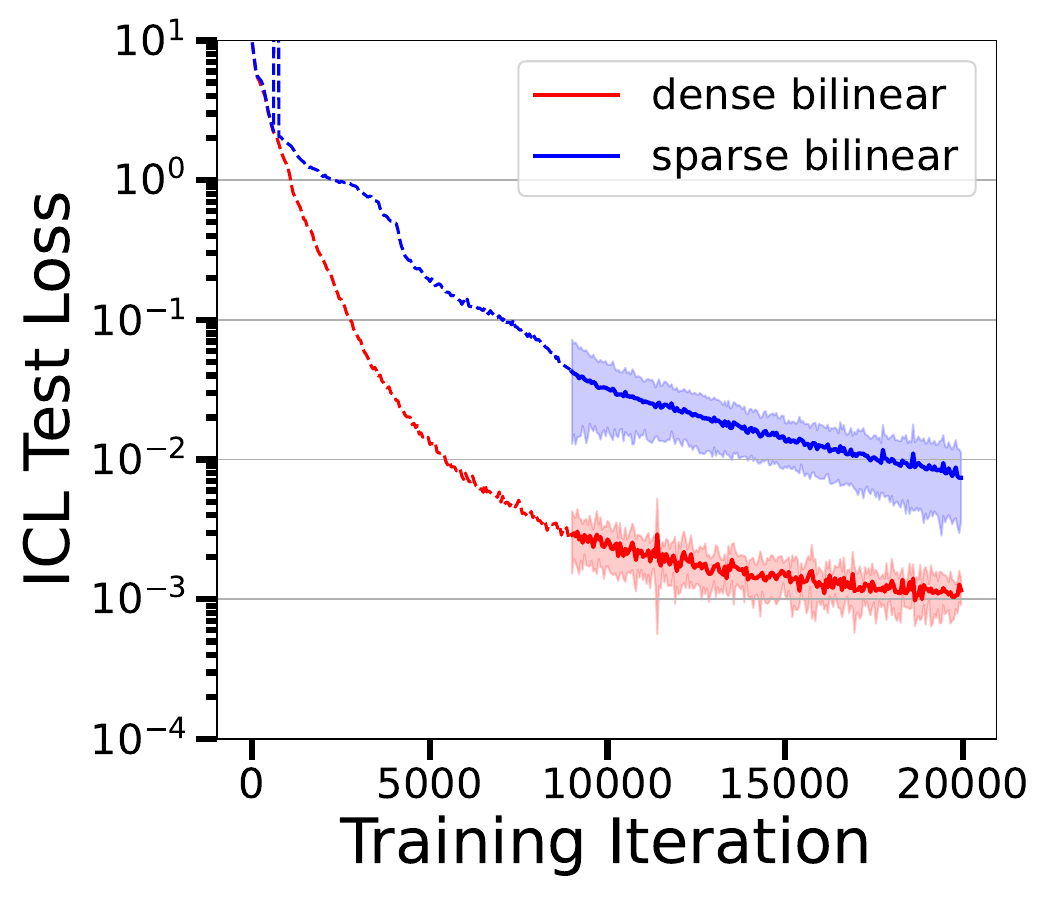}
    \caption{Loss curve (over 5 trials) for the cubic ICL task.}
\end{figure}

\section{Experiment Details}
\label{sec:experiment-details}
In our experiments, we trained our models using the Adam optimizer with a learning rate of $10^{-3}$ and no weight decay.
For each iteration, we independently sample a batch of 4000 prompts.

\begin{list}{{\tiny $\blacksquare$}}{\leftmargin=1em}
\setlength{\itemsep}{2pt}
    \item For the experiment on quadratic ICL, we trained a linear Transformer with 6 layers and $2L$-layer bilinear Transformer for $L = \{2, 4, 6\}$.
    For the other parameters, we used $n=200, \bar{d}=12$.
    We performed two sets of experiments with quadratic target functions of 3 variables ($d=3$) and 4 variables ($d=4$), where all other parameters are fixed to be the same.
    \item For the experiment in Figure~\ref{fig:gen-plot}, we trained an 8-layer bilinear Transformer, which implements 4 steps of block-coordinate descent.
    We used $d=3, \bar{d}=12$ and varied the prompt length at training $n_{\rm train} \in \{100, 150, 200, 280\}$.
    \item For the experiment on cubic ICL, we trained a 5-layer linear Transformer and two 10-layer bilinear Transformers with $n=100, d = 4, \bar{d}=15$.
    For the bilinear Transformer model with sparse bilinear layers, we only train on the free weights as specified by \eqref{equ:sparse-bilinear}, while the remaining entries of bilinear layer parameters are fixed to be zero.
    \item For the experiment in Figure~\ref{fig:verification}, we directly use the bilinear Transformer block construction given in Appendix~\ref{sec:proof-bilinear-construction} without any training. We compute the test loss by drawing 100 prompts with $d=4, \bar{d}=15$.
\end{list}

\noindent All of the experiments were conducted on a single Nvidia V100 GPU.
The total runtime of our experiments are about 2 days.
 
\end{document}